\documentclass[11pt]{article}
\usepackage[dvipsnames, svgnames, x11names]{xcolor}
\usepackage[utf8]{inputenc}
\usepackage[authoryear, round]{natbib}
\usepackage[backref, colorlinks, citecolor=Blue]{hyperref}
\usepackage{geometry}
\usepackage{booktabs}
\usepackage{threeparttable}
\usepackage{footmisc}
\usepackage{amsmath}
\usepackage{amssymb}
\usepackage{amsthm}
\usepackage{bbm}
\usepackage{graphicx}
\usepackage{courier}
\usepackage[ruled,linesnumbered]{algorithm2e}
\usepackage{appendix}
\usepackage{enumerate}
\usepackage[shortlabels]{enumitem}
\usepackage{multirow}
\usepackage{authblk}
\usepackage[T1]{fontenc}

\numberwithin{equation}{section}
\numberwithin{figure}{section}
\numberwithin{table}{section}

\newtheorem{theorem}{Theorem}[section]
\newtheorem{lemma}[theorem]{Lemma}
\newtheorem{proposition}[theorem]{Proposition}
\newtheorem{corollary}[theorem]{Corollary}

\theoremstyle{definition}
\newtheorem{assumption}{Assumption}[section]

\theoremstyle{remark}
\newtheorem{remark}{Remark}[section]

\usepackage{amsmath,amsfonts,bm}

\def\rd{{\mathrm{d}}}

\DeclareMathAlphabet{\mathsfit}{\encodingdefault}{\sfdefault}{m}{sl}
\SetMathAlphabet{\mathsfit}{bold}{\encodingdefault}{\sfdefault}{bx}{n}

\def\gA{{\mathcal{A}}}

\def\gE{{\mathcal{E}}}
\def\gF{{\mathcal{F}}}
\def\gG{{\mathcal{G}}}

\def\gL{{\mathcal{L}}}

\def\gN{{\mathcal{N}}}

\def\gS{{\mathcal{S}}}
\def\gT{{\mathcal{T}}}

\def\0{{\bf 0}}
\def\1{{\bf 1}}

\def\HB{{\mathbb H}}

\def\PB{{\mathbb P}}

\def\SB{{\mathbb S}}

\def\bsA{{\boldsymbol A}}
\def\bsB{{\boldsymbol B}}

\def\bsG{{\boldsymbol G}}
\def\bsH{{\boldsymbol H}}
\def\bsI{{\boldsymbol I}}

\def\bsM{{\boldsymbol M}}

\def\bsS{{\boldsymbol S}}

\def\bsU{{\boldsymbol U}}
\def\bsV{{\boldsymbol V}}

\def\bsX{{\boldsymbol X}}

\def\bsg{{\boldsymbol g}}

\def\bsr{{\boldsymbol r}}

\def\bsu{{\boldsymbol u}}
\def\bsv{{\boldsymbol v}}

\def\bsx{{\boldsymbol x}}
\def\bsy{{\boldsymbol y}}
\def\bsz{{\boldsymbol z}}

\def\bsvareps{{\boldsymbol \varepsilon}}

\def\bsSigma{{\boldsymbol \Sigma}}
\def\bsLambda{{\boldsymbol \Lambda}}

\def\bsDelta{{\boldsymbol \Delta}}
\def\bszeta{{\boldsymbol \zeta}}

\def\bsGamma{{\boldsymbol \Gamma}}
\def\bslambda{{\boldsymbol \lambda}}

\def\var{\mathrm{var}}

\def\tr{\mathrm{tr}}

\def\diag{\mathrm{diag}}

\newcommand{\R}{\mathbb{R}}
\newcommand{\E}{\mathbb{E}}

\newcommand{\Vol}{\operatorname{Vol}}
\newcommand{\Real}{\operatorname{Re}}

\def\indicator{{\mathbbm{1}}}

\title{Asymptotic Time-Uniform Inference for Parameters in Averaged Stochastic Approximation}
\author{
    Chuhan Xie\footnote{School of Mathematical Sciences, Peking University, Email: \href{ch\_xie@pku.edu.cn}{ch\_xie@pku.edu.cn}} 
    \quad Kaicheng Jin\footnote{School of Mathematical Sciences, Peking University, Email: \href{2301110071@pku.edu.cn}{2301110071@pku.edu.cn}} 
    \quad Jiadong Liang\footnote{School of Mathematical Sciences, Peking University, Email: \href{jdliang@pku.edu.cn}{jdliang@pku.edu.cn}} 
    \quad Zhihua Zhang\footnote{School of Mathematical Sciences, Peking University, Email: \href{2301110071@pku.edu.cn}{zhzhang@math.pku.edu.cn}}
}
\date{}

\begin{document}

\maketitle
\vspace{-2em}
\begin{abstract}
  We study time-uniform statistical inference for parameters in stochastic approximation (SA), which encompasses a bunch of applications in optimization and machine learning. To that end, we analyze the almost-sure convergence rates of the averaged iterates to a scaled sum of Gaussians in both linear and nonlinear SA problems. We then construct three types of asymptotic confidence sequences that are valid uniformly across all times with coverage guarantees, in an asymptotic sense that the starting time is sufficiently large. These coverage guarantees remain valid if the unknown covariance matrix is replaced by its plug-in estimator, and we conduct experiments to validate our methodology. 
\end{abstract}

\section{Introduction}

Traditional statistical inference for parameters is based on asymptotic/non-asymptotic coverage guarantees at fixed sample sizes. Such guarantees tend to become problematic in sequential experimental design due to the issue of ``peeking'', i.e., experimenters deciding whether to collect more data for further experiments after looking at current results in order to make the outcome more significant \citep{feller1940statistical,anscombe1954fixed,robbins1952some}.
Recently, there has been an emerging literature on safe anytime-valid inference \citep{johari2022always,howard2021time,grunwald2020safe,shafer2021testing,pace2020likelihood,ramdas2023game} that partially solve such  problems. At a high level, they utilize a supermartingale method coupled with  Ville's inequality \citep{Ville1939} to obtain time-uniform coverage guarantees for the quantities of interest, which is also equivalent to coverage guarantees for arbitrary stopping times \citep{ramdas2020admissible}:
\begin{align*}
    \PB(\forall t\geq 1 \colon \theta_t \in \operatorname{CI}_t ) \geq 1 - \alpha \quad \Longleftrightarrow \quad \PB(\theta_\tau \in \operatorname{CI}_\tau) \geq 1- \alpha \quad \text{for all stopping times } \tau>0.
\end{align*}
With such guarantees, the corresponding statistical inference procedure remains valid even in the presence of peeking, since the experimenter's decision of whether to stop conducting further experiments can always be regarded as a stopping time.
This technique has been applied to a series of traditional statistical tasks including mean estimation and hypothesis tests \citep{darling1968some,howard2021time,waudby2024estimating,ramdas2023game}. 

This work takes a step forward to study time-uniform statistical inference in the stochastic approximation (SA) framework, which refers to a wide class of online iterative algorithms and is prevalent in modern machine learning applications such as stochastic gradient descent (SGD) and reinforcement learning (RL) due to memory and efficiency superiority over offline methods. Specifically, the SA iterates $\{\bsx_t\}_{t\geq 0}$ follow the recursive stochastic updates as \eqref{eq: sa} below. Despite recent progress on statistical inference for SA \citep{chen2020statistical,fang2018online,su2023higrad,lee2022fast}, none of them provide time-uniform results and therefore they still suffer from the peeking issue. The goal of this paper is to provide asymptotic time-uniform coverage guarantees of the following form:
\begin{align*}
    \lim_{m \rightarrow \infty} \PB(\forall t\geq m \colon \Bar{\bsx}_t - \bsx^* \in \operatorname{CI}_t) \geq 1-\alpha,
\end{align*}
where $\Bar{\bsx}_t := t^{-1}\sum_{j=1}^t \bsx_t$ is the averaged iterate and $\bsx^*$ is the true parameter. As a direct consequence, for sequential experiments on e-commerce platforms or in clinical trials where the parameters of interest are updated iteratively, our results provide confidence sequences for the true parameter with asymptotically valid error control regardless of the data collection strategy of experimenters.

\subsection{Our Contributions}

Our theoretical contributions mainly lie in the following three parts.
\begin{enumerate}[(i)]
    \item We establish an almost-sure Gaussian approximation result for the averaged iterates $\{\Bar{\bsx}_t\}_{t\geq 0}$ to a scaled sum of i.i.d. Gaussians with precise convergence rates (Theorem \ref{thm: strong approx}), and provide a refined analysis on the relationship between the hyperparameters and these rates (Section \ref{sec: rate anal}) as well as the optimal rates (Section \ref{sec: optimal rate}) in both linear and nonlinear cases.
    \item We devise three confidence sequence boundaries for the scaled sum of i.i.d.\ Gaussians and provide time-uniform coverage guarantees (Section \ref{sec: boundary}).
    \item For practical inference, we derive the almost-sure convergence for the plug-in estimator of the unknown covariance matrix that is used in constructing confidence sequences (Section \ref{sec: boundary approx}), and prove that asymptotic time-uniform coverage guarantees still hold when replacing this covariance matrix with its plug-in analogue (Section \ref{sec: acs result}).
\end{enumerate}

We perform numerical experiments in Section \ref{sec: experiment} to validate our methodology. Our idea stems from the time-uniform central limit theory proposed in \cite{waudby2021time}, which highlights the necessity of ``almost-sure'' convergence of all quantities of interest, instead of usual ``in-probability'' or ``in-law'' ones. This requirement leads us to derive brand new convergence results mentioned above in the SA framework.

\subsection{Related Work}

\paragraph{Anytime-valid inference.} The idea of confidence sequences is initiated by \cite{darling1967confidence}, which provided a sequential estimation approach for the median. A more recent work \cite{howard2021time} gave comprehensive analysis and constructions of confidence sequences for the mean based on time-uniform concentration bounds from \cite{howard2020time}. Similar techniques were adopted in \cite{howard2022sequential} for quantile estimation and in \cite{waudby2024estimating} for bounded mean estimation. Another line of research, nonparametric sequential testing, goes back to \cite{wald1945sequential}. Based on the test martingale method from \cite{shafer2011test}, \cite{grunwald2020safe} developed a theory of safe hypothesis testing and proposed general methods for constructing tests with optimal powers. Subsequent works include \cite{ramdas2020admissible} for testing symmetry, \cite{ramdas2022testing} for testing exchangeability, etc.

\paragraph{Statistical inference for stochastic approximation.} 
Based on the seminal work of \cite{polyak1992acceleration}, \cite{chen2020statistical} studied asymptotic normality of averaged iterates from SGD, and proposed the plug-in and batch-means estimators for the asymptotic covariance matrix in order to perform statistical inference on model parameters. 
\cite{zhu2023online} developed a fully online batch-means method which improves computation and memory efficiency over the original one in \cite{chen2020statistical}. \cite{fang2018online} proposed an online bootstrap procedure for the estimation of confidence intervals via a number of randomly perturbed SGD iterates.
\cite{su2018uncertainty} designed a hierarchical incremental gradient descent method termed HiGrad that hierarchically split SGD updates into multiple threads and constructed a $t$-based confidence interval for the parameter. \cite{lee2022fast} leveraged insights from time series regression in econometrics \citep{abadir1997two,kiefer2000simple} and constructed asymptotic pivotal statistics for parameter inference via random scaling, which was extended to different machine learning scenarios in \cite{li2022statistical,li2023statistical,li2023online}.
\cite{chen2021statistical1,chen2021statistical2,chen2022online} studied online statistical inference in the contextual bandit settings and analyzed the performance and efficiency of weighted versions of SGD.





\section{Problem Setup}
\label{sec: problem}

We begin by introducing the problem setup. We are concerned with solving the root-finding problem $\bsg(\bsx^*) = \0$, where $\bsg\colon \R^d \rightarrow \R^d$ is expressed as an expectation over the data point $\xi$; i.e., $\bsg(\bsx) = \E[\bsG(\bsx,\xi)]$. Supposing we only have access to a sequence of i.i.d. data points $\xi_1,\xi_2,\dots$, a typical stochastic approximation (SA) algorithm is given by the following $d$-dimensional recursion:
\begin{align}
\label{eq: sa}
    \bsx_{t+1} = \bsx_t - \eta_t \bsG(\bsx_t,\xi_{t+1}),\quad t=0,1,2,\dots
\end{align}
Starting from an initializer $\bsx_0$, the iterates $\{\bsx_t\}_{t\geq 0}$ will converge to the unique root $\bsx^*$ and possess favorable statistical properties under certain assumptions that we are about to display.


\begin{assumption}[Lyapunov function]
\label{ass: lyapunov}
    There exist a differentiable function $V \colon \R^d \rightarrow \R$ and a constant $\delta_V > 0$ such that for all $\|\bsx - \bsx^*\| \leq \delta_V$ and $\bsy,\bsz\in \R^d$, the following conditions hold.
    \begin{enumerate}[(i)]
        \item $V(\0)=0$. \label{ass: lyapunov 1}
        \item $V(\bsx - \bsx^*) \geq \mu_V \|\bsx - \bsx^*\|^2$ for some $\mu_V > 0$. \label{ass: lyapunov 2}
        \item $\|\nabla V(\bsy) - \nabla V(\bsz)\| \leq L_V \|\bsy - \bsz\|$ for some $L_V > 0$. \label{ass: lyapunov 3}
        \item $\nabla V(\bsx - \bsx^*)^\top \bsg(\bsx) > 0$ for $\bsx \neq \bsx^*$. \label{ass: lyapunov 4}
        \item $\nabla V(\bsx - \bsx^*)^\top \bsg(\bsx) \geq \lambda_V V(\bsx - \bsx^*)$ for some $\lambda_V>0$. \label{ass: lyapunov 5}
    \end{enumerate}
\end{assumption}

\begin{assumption}[Local linearity]
\label{ass: local linear}
    There exist a matrix $\bsH\in \R^{d\times d}$ and constants $L_H>0$, $\delta_H>0$, $0<\lambda\leq 1$ such that for all $\|\bsx - \bsx^*\| \leq \delta_H$,
    \begin{align*}
        \| \bsg(\bsx) - \bsH(\bsx - \bsx^*) \| \leq L_H \|\bsx - \bsx^*\|^{1+\lambda},
    \end{align*}
    and that $\Real \lambda_i(\bsH) > 0$ for all $1\leq i\leq d$, where $\lambda_i(\cdot)$ denotes the $i$-th eigenvalue.
\end{assumption}

\begin{assumption}[Lipschitzness]
\label{ass: lipschitz}
    There exists a constant $L_G>0$ such that for any $\bsx,\bsy \in \R^d$,
    \begin{align*}
        \sqrt{\E\big( \| \bsG(\bsx,\xi) - \bsG(\bsy,\xi) \|^2 \big) } \leq L_G \|\bsx - \bsy\|.
    \end{align*}
\end{assumption}

\begin{assumption}[Step size]
\label{ass: step size}
    The step size is $\eta_t = \eta_0 t^{-a}$ for some constant $1/2 < a < 1$.
\end{assumption}

\begin{assumption}[Noise]
\label{ass: noise}
    For each $t\geq 1$, let $\bsvareps_t := \bsg(\bsx_{t-1}) - \bsG(\bsx_{t-1}, \xi_{t})$. The martingale difference sequence $\{\bsvareps_t\}_{t\geq 1}$ satisfies the following conditions.
    \begin{enumerate}[(i)]
        \item For all $t\geq 1$, $\E(\|\bsvareps_t\|^2\mid \gF_{t-1}) + \|\bsg(\bsx_{t-1})\|^2 \leq L_\varepsilon (1 + \|\bsx_{t-1} - \bsx^*\|^2)$ almost surely for some constant $L_\varepsilon>0$. \label{ass: noise 1}
        \item \label{ass: noise 2} The following decomposition holds: $\bsvareps_t = \bsvareps_t(0) + \bszeta_t(\bsx_{t-1})$, where
        \begin{enumerate}[(a)]
            \item $\E(\bsvareps_t(0)\mid \gF_{t-1}) = 0$ almost surely;
            \item $\E(\bsvareps_t(0)\bsvareps_t(0)^\top \mid \gF_{t-1}) \xrightarrow[]{P} S$ as $t\rightarrow \infty$, where $\bsS\in \SB^{d}_+$ is a symmetric and positive definite matrix;
            \item $\sup_{t} \E(\|\bsvareps_t(0)\|^2 \indicator(\|\bsvareps_t(0)\| > C)\mid \gF_{t-1}) \xrightarrow[]{P} 0$ as $C\rightarrow \infty$;
            \item for all $t$ large enough, $\E(\|\bszeta_t(\bsx_{t-1})\|^2\mid \gF_{t-1}) \leq \delta(\bsx_{t-1} - \bsx^*)$ almost surely, where $\delta(\bsDelta) \rightarrow 0$ as $\bsDelta \rightarrow 0$.
        \end{enumerate}
        \item For all $t\geq 1$, $\E(\|\bsvareps_t\|^{2p}) \leq C_\varepsilon$ and $\E(\|\bsG(\bsx^*,\xi)\|^{2p}) \leq C_\varepsilon$ for some constants $C_\varepsilon$ and $p>1$. \label{ass: noise 3}
    \end{enumerate}
\end{assumption}

\begin{remark}
    Assumptions \ref{ass: lyapunov}-\ref{ass: noise} are analogous to Assumptions 3.1-3.4 of \cite{polyak1992acceleration}, which are standard in the SA literature. Assumptions \ref{ass: lyapunov}-\ref{ass: lipschitz} are restrictions on $\bsg(\bsx)$ and $\bsG(\bsx,\xi)$ to ensure convergence of the SA algorithm \eqref{eq: sa}. In Assumption \ref{ass: step size} we use a standard polynomial step size schedule that satisfies the Robbins-Monro conditions \citep{robbins1951stochastic}, i.e., $\sum_{t=0}^\infty \eta_t = \infty$ and $\sum_{t=0}^{\infty} \eta_t^2 < \infty$.
    Assumption \ref{ass: noise} contains conditions on the scale of noise and is essential to establishing the asymptotic behavior of the averaged iterates. The SA framework includes a bunch of examples including linear regression and logistic regression, which are studied in our experiments (Section \ref{sec: experiment}). 
\end{remark}



    

    

    

\section{Almost-Sure Gaussian Approximation}
\label{sec: strong approx}

In the following Theorem \ref{thm: strong approx}, we approximate the averaged iterate $\Bar{\bsx}_t - \bsx^*$ by a sum of i.i.d. Gaussians in an almost-sure sense, and characterize the corresponding approximation rate. Its detailed proof is deferred to Appendix \ref{sec: thm: strong approx}.

\begin{theorem}[Gaussian approximation]
    \label{thm: strong approx}
    Let $\epsilon>0$ be an arbitrarily small constant. Under Assumptions \ref{ass: lyapunov}-\ref{ass: noise}, there exists a sequence of $d$-dimensional i.i.d. Gaussian random vectors $\{\bsG_t\}_{t\geq 1}$ with mean zero and covariance $\bsH^{-1}\bsS \bsH^{-1}$ on the same (potentially enriched) probability space, such that
    \begin{align*}
        \Bar{\bsx}_t - \bsx^* = \frac{1}{t} \sum_{j=1}^t \bsG_j + o(r_{t,\epsilon} + w_{t,\epsilon,d}) \quad \text{a.s.},
    \end{align*}
    where 
    \begin{align*}
        &r_{t,\epsilon} = t^{-1}(\log t)^{\epsilon} + t^{-\frac{a(1+\lambda)}{2}}(\log t)^{1+\epsilon} + t^{\frac{1}{2p} - \frac{2-a}{2}}(\log t)^{\frac{1}{2p}+\epsilon} + t^{-\frac{1+a}{2}}(\log t)^{\frac{1}{2}+\epsilon}, \\
        &w_{t,\epsilon,1} = t^{-\frac{1}{2} - \frac{p-1}{4p}}(\log t)^{1+\frac{1}{4p} + \epsilon}, \quad w_{t,\epsilon,d} = t^{-\frac{1}{2} - \frac{p-1}{50dp}} (\log t)^{\frac{1}{50dp} + \epsilon} \quad \text{for } d>1.
    \end{align*}
\end{theorem}

Compared with the (functional) central limit theorems for averaged iterates in the literature \citep{polyak1992acceleration,chen2020statistical,chen2024online,lee2022fast,lee2024fast,chen2023sgmm,li2022statistical,li2023statistical,li2023online}, our Theorem \ref{thm: strong approx} provides an almost-sure version instead of a usual in-probability one under some mildly stronger assumptions (see Assumption \ref{ass: step size rate} below). In addition, we precisely characterize the approximation rate when $t\rightarrow \infty$, while previous works only verify that it is of order $o(t^{-1/2})$ for the sole purpose of establishing asymptotic normality.

\subsection{Proof Idea of Theorem \ref{thm: strong approx}}
    The high-level idea of its proof is to decompose $\Bar{\bsx}_t - \bsx^*$ into the following five parts \citep{polyak1992acceleration}: denoting $\bsA_{j}^{t-1} = \sum_{s=j}^{t-1} \left( \prod_{i=j+1}^s (\bsI - \eta_i \bsH) \right) \eta_j$, $\bsr_j = \bsH(\bsx_j - \bsx^*) - \bsg(\bsx_j)$, $\bsvareps_{j+1} = \bsg(\bsx_j) - \bsG(\bsx_j, \xi_{j+1})$ and $\Tilde{\bsvareps}_{j+1} = \bsg(\bsx^*) - \bsG(\bsx^*, \xi_{j+1})$, we have
    \begin{align*}
        \Bar{\bsx}_t - \bsx^* &= \underbrace{\frac{1}{t} \bsA_0^{t-1} (\bsI - \eta_0 \bsH) (\bsx_0 - \bsx^*)}_{\gS_1} + \underbrace{\frac{1}{t} \sum_{j=0}^{t-1} \bsA_j^{t-1} \bsr_j}_{\gS_2} + \underbrace{\frac{1}{t} \sum_{j=0}^{t-1} (\bsA_j^{t-1} - \bsH^{-1}) \bsvareps_{j+1}}_{\gS_3} \\
        &\quad + \underbrace{\frac{1}{t} \sum_{j=0}^{t-1} \bsH^{-1} (\bsvareps_{j+1} - \Tilde{\bsvareps}_{j+1})}_{\gS_4} + \underbrace{\frac{1}{t} \sum_{j=0}^{t-1} \bsH^{-1} \Tilde{\bsvareps}_{j+1}}_{\gS_5}.
    \end{align*}
    Here, $\bsA_j^{t-1}$ can be regarded as an approximation of $\bsH^{-1}$. Below we discuss how the different rates contained in $r_{t,\epsilon}$ and $w_{t,\epsilon,d}$ appear from the above five error terms.
    
    Put simply, $\gS_1$ is the error caused by inaccurate initialization and is of order $O(t^{-1})$. And $\gS_2$ is the error induced when we linearize the expected increment $\bsg(\bsx_j)$ of an update \eqref{eq: sa} by $\bsH(\bsx_j - \bsx^*)$. By Assumption \ref{ass: local linear}, when $\bsx_j$ is close to $\bsx^*$ (which is always true when $j\rightarrow \infty$), this linearization error will decay at a rate similar to that of $\|\bsx_j - \bsx^*\|^{1+\lambda}$, and we rigorously prove it is of order $o\big(t^{-\frac{a(1+\lambda)}{2}}(\log t)^{1+\epsilon}\big)$ in Lemma \ref{lem: as clt approx}. 
    
    $\gS_3$ characterizes the approximation error of $\bsA_j^{t-1}$ to $\bsH^{-1}$ and is of order $o\big(t^{\frac{1}{2p} - \frac{2-a}{2}}(\log t)^{\frac{1}{2p}+\epsilon}\big)$, as also proved in Lemma \ref{lem: as clt approx}. Compared to the corresponding proof in \cite{polyak1992acceleration} (see Lemma 2 and the proof of Theorem 1 therein), in order for the almost-sure convergence rate of $\gS_3$ to be faster than $t^{-1/2}$, we additionally need a finite $2p$-th moment condition on the noise, which we formulate as Assumption \ref{ass: noise} \eqref{ass: noise 3}. This condition also appears in previous works \citep{zhu2021constructing,lee2022fast,li2022statistical,li2023online} concerning weak convergence of the scaled iterates to a Brownian motion (i.e., functional central limit theorems), and is also essential to our proof.

    $\gS_4$ characterizes the error induced when we calibrate the increments of the martingale $\sum_{j=0}^{t-1} \bsH^{-1} \bsvareps_{j+1}$ to i.i.d. random variables, sharing the same randomness $\{\xi_t\}_{t\geq 1}$ and the same limiting covariance $\bsH^{-1}\bsS \bsH^{-1}$. This is easily achieved by considering the martingale $\sum_{j=0}^{t-1} \bsH^{-1} \Tilde{\bsvareps}_{j+1}$. Using the averaged Lipschitzness in Assumption \ref{ass: lipschitz}, again we link the convergence rate of $\gS_4$ to that of $\|\bsx_t - \bsx^*\|^2$ and derive a rate of $o\big(t^{-\frac{1+a}{2}}(\log t)^{\frac{1}{2}+\epsilon}\big)$ in Lemma \ref{lem: strong approx}.

    Finally, $\gS_5$ is approximated by i.i.d. Gaussians with the same mean and covariance, using almost-sure invariance principles from \cite{strassen1967almost, philipp1986note}. For one-dimensional SA problems with $d=1$, the approximation error is of order $o\big( t^{-\frac{1}{2} - \frac{p-1}{4p}}(\log t)^{1+\frac{1}{4p} + \epsilon} \big)$ \citep{strassen1967almost}. For high-dimensional SA problems with $d>1$, the error rate is slightly worsened to $o\big( t^{-\frac{1}{2} - \frac{p-1}{50dp}}(\log t)^{\frac{1}{50dp} + \epsilon} \big)$ \citep{philipp1986note}.

    Combining the convergence rates of the above five parts finally yields the overall strong approximation rate $o(r_{t,\epsilon} + w_{t,\epsilon,d})$ for arbitrarily small $\epsilon>0$.

\subsection{Rate Analysis}
\label{sec: rate anal}
Based on Theorem \ref{thm: strong approx}, we now perform a detailed analysis on the behavior of approximation rates. 
In Figure \ref{fig: as_rate} below, we plot the negative exponents of these rates separately with respect to $a$ (which is the negative exponent of the step size as defined in Assumption \ref{ass: step size}) for both linear and nonlinear SA problems, resulting in several new findings.

\begin{figure}[ht]
    \centering
    \includegraphics[scale=0.42]{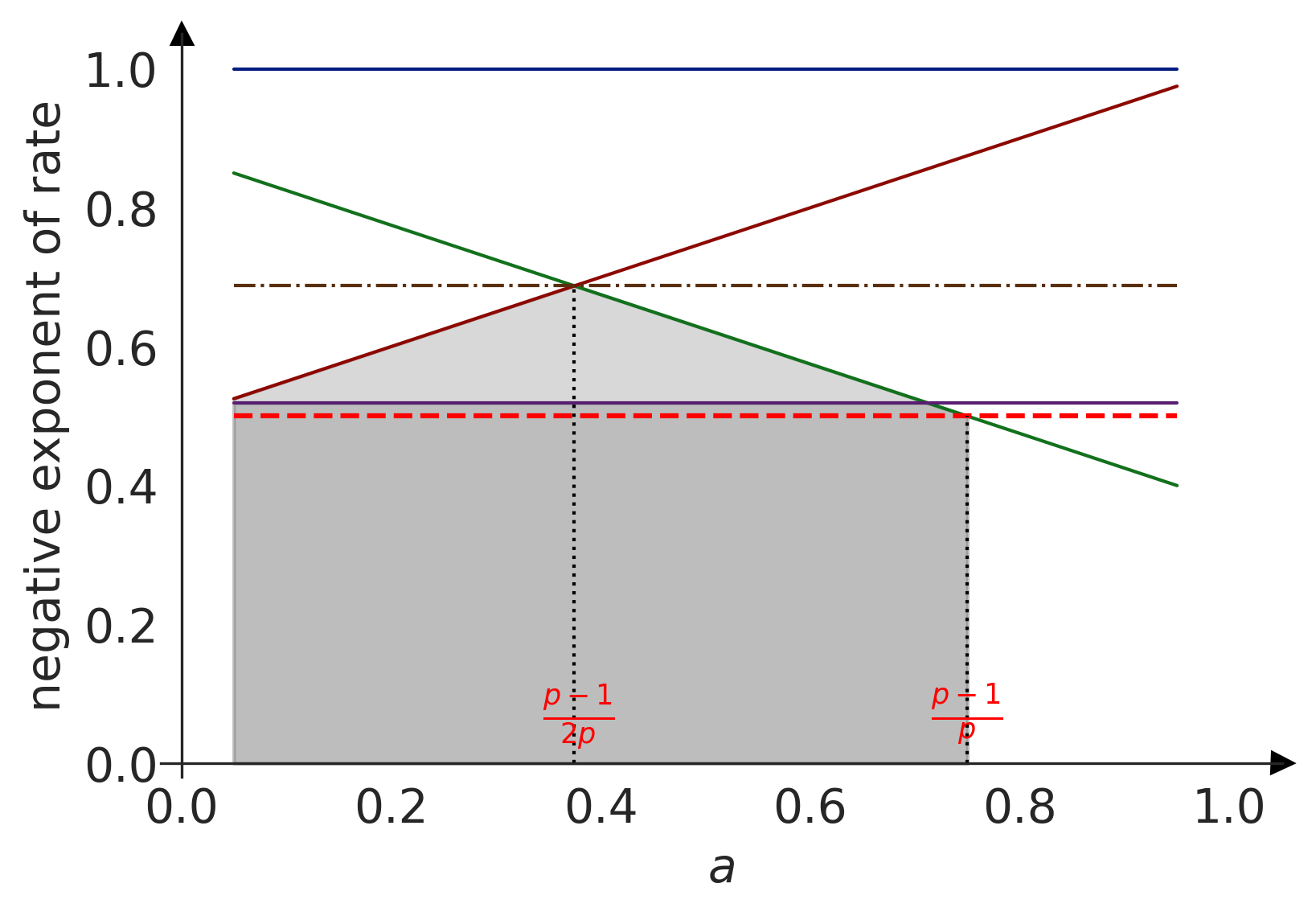}
    \includegraphics[scale=0.42]{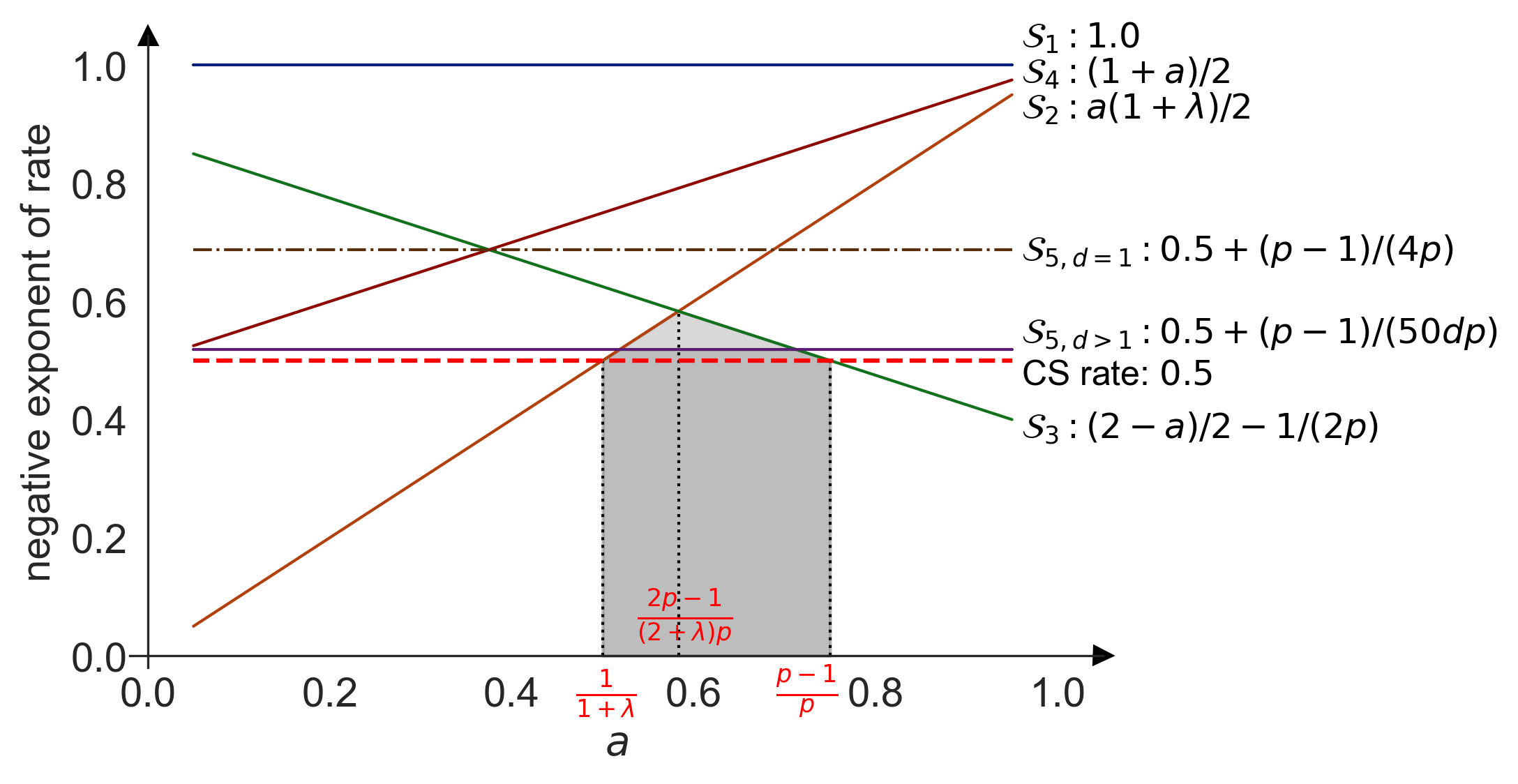}
    \caption{Almost-sure approximation rates for  all error terms with respect to the choice of $a$ in both linear (Left) and nonlinear (Right) SA settings. Respectively, $\gS_1$ represents the initialization error, $\gS_2$ represents the linearization error which is absent in the linear case, $\gS_3$ represents the matrix-inverse approximation error, $\gS_4$ represents the covariance calibration error, and $\gS_5$ represents the error due to the almost-sure invariance principles.}
    \label{fig: as_rate}
\end{figure}

Firstly, the linearization error $\gS_2$ and the covariance calibration error $\gS_4$ decay faster for larger $a$, while the matrix-inverse approximation error $\gS_3$ decays slower for larger $a$. This is due to the fact that both $\gS_2$ and $\gS_4$ are controlled by the $L^2$ error $\E\|\bsx_t - \bsx^*\|^2$. For large $t>0$, $\E\bsx_t$ gets very close to $\bsx^*$ so the variance of the iterate dominates its bias in $\E\|\bsx_t - \bsx^*\|^2$. Since its variance is proportional to that of the increment $-\eta_{t} \bsG(\bsx_{t},\xi_{t+1})$ in \eqref{eq: sa}, it follows that larger $a$ implies smaller $\eta_t$ and thus faster convergence of $\E\|\bsx_t - \bsx^*\|^2$ (see Lemma \ref{lem: as l2 conv} for details). On the other hand, $\gS_3$ purely depends on the approximation rate of $\bsA_j^t$ to $\bsH^{-1}$, which is unrelated to the randomness in data points $\{\xi_t\}_{t\geq 0}$. From the definition of $\bsA_j^t$ as well as Lemma \ref{lem: convergence matrix}, smaller $\eta_t$ will, on the contrary, impede the convergence of $\bsA_j^t$.

Secondly, we hope the overall approximation rate $r_{t,\epsilon}+w_{t,\epsilon,d}$ to be of order $\Tilde{o}(t^{-1/2})$, so that the invariance principle and the law of the iterated logarithm for $\frac{1}{t}\sum_{j=1}^t \bsG_j$ will also hold for $\Bar{\bsx}_t - \bsx^*$. To meet this requirement, the rate of step sizes $a$ must satisfy $0<a<\frac{p-1}{p}$ in the linear case and $\frac{1}{1+\lambda} < a < \frac{p-1}{p}$ in the nonlinear case, the latter of which implicitly requires $p>\frac{1+\lambda}{\lambda}$. In the special case where $\lambda=1$ and $p=\infty$, the feasible range for $a$ becomes $0<a<1$ and $1/2 < a < 1$, respectively, matching the corresponding conditions in the literature \citep{polyak1992acceleration}. In the sequel, we add these conditions for $a$ to pave the way for further development of time-uniform inference.

\begin{assumption}[Rate condition]
\label{ass: step size rate}
    The constants $\lambda, p, a, L_H$ satisfy one of the following conditions:
    \begin{enumerate}[(i)]
        \item (Linear SA) $L_H = 0$ and $0<a<\frac{p-1}{p}$.
        \item (Nonlinear SA) $p>\frac{1+\lambda}{\lambda}$ and $\frac{1}{1+\lambda} < a < \frac{p-1}{p}$.
    \end{enumerate}
\end{assumption}

\subsection{Optimal Approximation Rates}
\label{sec: optimal rate}
In this section, we study the best possible approximation rate of the overall errors in Theorem \ref{thm: strong approx}. Note that $w_{t,\epsilon,d}$ stems from the existing almost-sure invariance principles \citep{strassen1967almost,philipp1986note} that are generally hard to improve, so we mainly focus on improving $r_{t,\epsilon}$. According to Figure \ref{fig: as_rate}, for linear SA problems, the best possible approximation rate is attained at $a = \frac{p-1}{2p}$, in which case $r_{t,\epsilon} = \widetilde{O}\Big(t^{-\frac{3p-1}{4p}}\Big)$; for nonlinear SA problems, the best possible approximation rate is attained at $a = \frac{2p-1}{(2+\lambda)p}$, in which case $r_{t,\epsilon} = \widetilde{O}\Big(t^{-\frac{(1+\lambda)(2p-1)}{2(2+\lambda)p}}\Big)$. In the special case where $\lambda=1$ and $p=\infty$, the optimal approximation rate becomes $r_{t,\epsilon} = \widetilde{O}(t^{-3/4})$ for the linear case and $r_{t,\epsilon} = \widetilde{O}(t^{-2/3})$ for the nonlinear case, respectively. These are brand new results that characterize the almost-sure convergence rates of the averaged iterates to a scaled sum of i.i.d. random vectors:
\begin{align}
\label{eq:opt app rate}
    \Bar{\bsx}_t - \bsx^* = \frac{1}{t} \sum_{j=1}^t \bsH^{-1}\Tilde{\bsvareps}_j + \Tilde{o}(r_{t}) \quad \text{a.s.}
\end{align}

\begin{corollary}[Linear SA]
    Suppose Assumption \ref{ass: local linear} holds with $L_H = 0$, i.e., the problem is linear. Under Assumptions \ref{ass: lyapunov}-\ref{ass: noise} and \ref{ass: step size rate}, if we set $a=\frac{p-1}{2p}$, then \eqref{eq:opt app rate} holds with $r_{t} = t^{-\frac{3p-1}{4p}}$.
    In particular, if $\lambda=1$, $p=\infty$, and we set $a=1/2$, then $r_t = t^{-3/4}$. The above approximation rates are optimal in the choice of $a$.
\end{corollary}

\begin{corollary}[Nonlinear SA]
    Suppose Assumption \ref{ass: local linear} holds with $L_H > 0$, i.e., the problem is nonlinear. Under Assumptions \ref{ass: lyapunov}-\ref{ass: noise} and \ref{ass: step size rate}, if we set $a=\frac{2p-1}{(2+\lambda)p}$, then \eqref{eq:opt app rate} holds with $r_t = t^{-\frac{(1+\lambda)(2p-1)}{2(2+\lambda)p}}$.
    In particular, if $\lambda=1$, $p=\infty$, and we set $a=2/3$, then $r_t = t^{-2/3}$. The above approximation rates are optimal in the choice of $a$.
\end{corollary}

\section{Construction of Asymptotic Confidence Sequences}

Having established the almost-sure Gaussian approximation result, it is reasonable to believe that confidence sequences for the well-behaved sum of Gaussians will also serve as confidence sequences for the averaged iterates in some asymptotic sense, since as $t\rightarrow \infty$ the above two behaves quite the same except for negligible terms. In this section, we follow this idea to design such asymptotic confidence sequences for the averaged iterates, and provide suitable error guarantees for them.

\subsection{Confidence Sequences for the Approximating Process}
\label{sec: boundary}

We consider three different types of confidence sequences for the mean of multivariate Gaussians, which serve as building blocks for time-uniform inference on averaged iterates. Their specific forms are displayed in Table \ref{tab: cs}, where we denote $\{\bsM_t\}_{t\geq 1}$ as the mean of $d$-dimensional i.i.d. Gaussian random vectors with mean zero and covariance $\bsV$. Detailed derivation is deferred to Appendix \ref{app: proof cs}.

\begin{table}[ht]
    \centering
    \caption{Confidence sequences for the mean of multivariate Gaussians}
    \label{tab: cs}
    \begin{tabular}{lll}
    \toprule
       Notation & Form & Reference \\
    \midrule
       $C_{\bsV}^{\text{LIL-UB}}(t)$ & $\left\| \bsV^{-\frac{1}{2}}\bsM_t \right\|_\infty \leq 1.7 \sqrt{\frac{ \log\log(2t) + 0.72 \log\left(\frac{10.4d}{\alpha} \right) }{t}}$ & \cite{howard2021time} Eq. 3.4 \\
       $C_{\bsV}^{\text{GM}}(t)$ & $\left\| \bsV^{-\frac{1}{2}}\bsM_t \right\|_2 \leq \sqrt{\frac{\left\{1 + \frac{t_0}{t\lambda^*}\right\} \left\{d\log\left(1 + \frac{t\lambda^*}{t_0} \right) + 2\log\left(\frac{1}{\alpha}\right) \right\} }{t} }$ & \cite{howard2021time} Eq. 3.7 \\
       \multirow{2}{*}{$C^{\text{LIL-EN}}_{\bsV}(t)$} & $\left\| \bsV^{-\frac{1}{2}}\bsM_t \right\|_2 \leq \frac{2}{1-\epsilon} \sqrt{\frac{1.4\log\log(2t\kappa(\bsV)) + L_{d,\epsilon,\alpha}(\bsV)}{ t}} $ 
       & \multirow{2}{*}{\cite{whitehouse2023timeuniform} Cor. 4.3} \\
        & $L_{d,\epsilon,\alpha}(\bsV) := \log\left( \frac{5.2C_d}{\alpha}\right) + (d-1)\log\left(\frac{3\sqrt{\kappa(\bsV)}}{\epsilon}\right)$ & \\
    \bottomrule
    \end{tabular}
\end{table}

\paragraph{LIL boundary with union bounds, $C_{\bsV}^{\text{LIL-UB}}(t)$.} This bound 
comes from Equation 3.4 of \cite{howard2021time}, which is called the law of the iterated logarithm (LIL) boundary because this boundary scales as $O(\sqrt{t^{-1}\log\log t})$. To accommodate the multivariate case in out settings, we use a union bound over $d$ dimensions of the standardized process $\bsV^{-\frac{1}{2}}\bsM_t$, which results in a $L^\infty$-type confidence sequence (see Appendix \ref{app: lil ub} for details).

\paragraph{Gaussian mixture boundary, $C_{\bsV}^{\text{GM}}(t)$.} This $L^2$-type bound comes from Equation 3.7 of \cite{howard2021time} and originally dates back to \cite{robbins1970boundary,lai1976boundary}. In the formula of $C_{\bsV}^{\text{GM}}(t)$, $t_0\geq 1$ is a user-specified time, and $\lambda^* = -W_{-1}(-\alpha^2/e)$ where $W_{-1}(x)$ is the lower branch of the Lambert $W$ function, i.e., the most negative real-valued solution in $z$ to $ze^z = x$. Such specific form of $C_{\bsV}^{\text{GM}}(t)$ results in the minimal volume of the confidence region at time $t_0$ among all Gaussian mixture boundaries (see Appendix \ref{app: gauss mix} for details).

\paragraph{LIL boundary with $\varepsilon$-nets, $C^{\text{LIL-EN}}_{\bsV}(t)$.} This $L^2$-type bound comes from \cite{whitehouse2023timeuniform} and scales as $O(\sqrt{t^{-1}\log\log t})$. Unlike the union bound technique used for $C_{\bsV}^{\text{LIL-UB}}(t)$, here $C^{\text{LIL-EN}}_{\bsV}(t)$ is derived via the $\varepsilon$-net arguments (e.g., Chapter 5 of \cite{wainwright2019high}). The boundary contains a user-specified $\epsilon\in (0,1)$, the conditional number $\kappa(\bsV)$, and a constant $C_d = \frac{d 2^d \Gamma(\frac{d+1}{2})}{\pi^{\frac{d-1}{2}}}$. All other hyperparameters in the original form (Corollary 4.3 of \cite{whitehouse2023timeuniform}) are set to be the same as those in $C^{\text{LIL-UB}}_{\bsV}(t)$ for ease of comparison (See Appendix \ref{app: lil en} for details).

\begin{figure}[ht]
    \centering
    \includegraphics[scale=0.95]{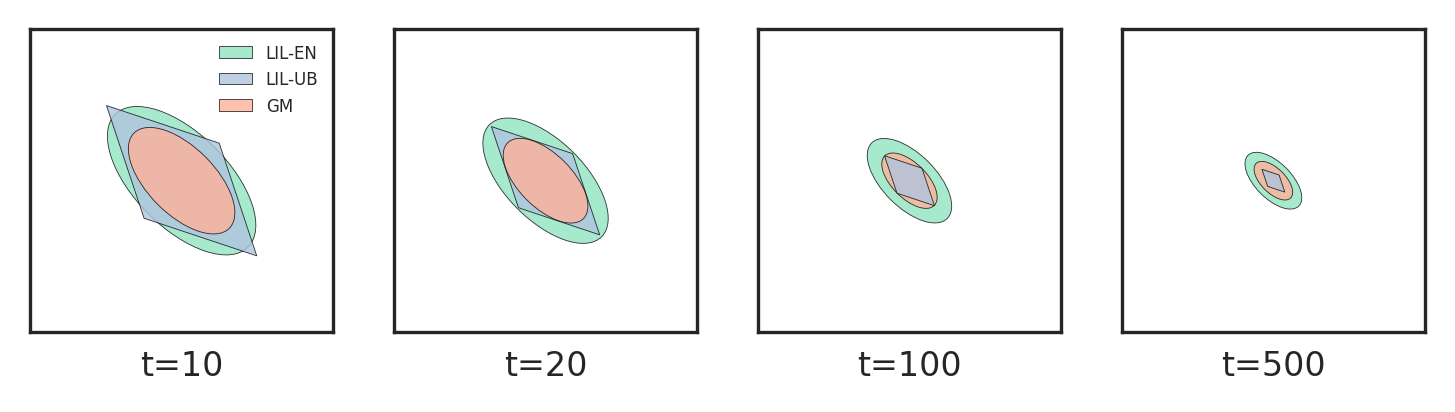}
    \caption{The shapes of three confidence regions at time $t\in \{10, 20, 100, 500\}$ in a 2D example.}
    \label{fig: cs_size}
\end{figure}

Figure \ref{fig: cs_size} plots the above three confidence regions at different times in a 2D example. LIL-EN and GM boundaries are ellipsoids while LIL-UB boundaries are polygons. GM boundaries are generally smaller than LIL-UB boundaries when $t$ is small but larger when $t$ is large. This conforms to the $O(\sqrt{t^{-1}\log t})$ rate of the former which grows faster than the $O(\sqrt{t^{-1}\log\log t})$ rate of the latter. LIL-EN boundaries are always the largest in this illustration.

The following Proposition \ref{prop: general boundary gaussian} formally states the time-uniform coverage property for these confidence sequences.

\begin{proposition}[Time-Uniform Coverage for Gaussians]
\label{prop: general boundary gaussian}
    Let $\bsG_1,\bsG_2,\dots$ be i.i.d. Gaussian random vectors with mean zero and covariance $\bsH^{-1}\bsS \bsH^{-1}$, and let $\bsM_t := \frac{1}{t}\sum_{j=1}^t \bsG_j$. Then for any $\clubsuit \in \{ \text{LIL-UB}, \text{GM}, \text{LIL-EN} \}$ and any $\alpha \in (0,1)$, it holds that
    \begin{align*}
        \PB \left( \forall t\geq 1\colon \bsM_t \in C_{\bsH^{-1}\bsS \bsH^{-1}}^{\clubsuit}(t) \right) \geq 1-\alpha.
    \end{align*}
\end{proposition}

\begin{remark}
    The results above assume simultaneous inference on all $d$ coordinates. To perform inference on only a subset $\gS\subseteq [d]$ of coordinates, it suffices to replace $\bsM_t$ with $[\bsM_t]_\gS$, $\bsH^{-1}\bsS\bsH^{-1}$ with $[\bsH^{-1}\bsS\bsH^{-1}]_{\gS,\gS}$, and $d$ with $|\gS|$ in the forms of confidence sequences in Proposition \ref{prop: general boundary gaussian} and Table \ref{tab: cs}.
\end{remark}

\subsection{Almost-Sure Convergent Variance Estimator}
\label{sec: boundary approx}

The confidence sequences proposed in Section \ref{sec: boundary} all rely on the covariance $\bsH^{-1}\bsS \bsH^{-1}$, which is usually unknown a priori. Consequently, to construct practical confidence sequences for time-uniform inference, we need to design consistent estimators for $\bsH^{-1}\bsS \bsH^{-1}$ in an online fashion. We mainly study the plug-in estimator \citep{chen2020statistical}, i.e., we estimate $\bsH$ and $\bsS$ with 
\begin{align*}
    \widehat{\bsH}_t := \frac{1}{t} \sum_{j=0}^{t-1} \Bar{\bsH}(\bsx_j, \xi_{j+1}),\quad \widehat{\bsS}_t := \frac{1}{t} \sum_{j=0}^{t-1} \bsG(\bsx_j, \xi_{j+1}) \bsG(\bsx_j, \xi_{j+1})^\top,
\end{align*}
and then use $\widehat{\bsH}_t^{-1}\widehat{\bsS}_t \widehat{\bsH}_t^{-1}$ to approximate $\bsH^{-1} \bsS \bsH^{-1}$.
\cite{chen2020statistical} proved $L^1$ convergence rates for $\widehat{\bsH}_t$ and $\widehat{\bsS}_t$, but these are insufficient for proving asymptotic time-uniform coverage guarantees \citep{waudby2021time}. Instead, we need almost-sure convergence for these estimators, which we present below.

\begin{assumption}[Unbiased estimate]
\label{ass: random est H}
    There exists a random function $\Bar{\bsH}(\cdot, \xi)\colon \R^{d} \rightarrow \R^{d\times d}$ such that $\bsH = \E \big( \Bar{\bsH}(\bsx^*,\xi) \big)$ and $\E \big( \|\Bar{\bsH}(\bsx^*,\xi)\|^{\Bar{p}} \big) < \infty$ for some $\Bar{p} > 1$.
\end{assumption}

\begin{assumption}[Lipschitzness]
\label{ass: lipschitz H}
    There exists a constant $L_{\Bar{H}} > 0$ such that for any $\bsx, \bsy \in \R^d$,
    \begin{align*}
        \sqrt{\E\big( \| \Bar{\bsH}(\bsx,\xi) - \Bar{\bsH}(\bsy,\xi) \|^2 \big) } \leq L_{\Bar{H}} \|\bsx - \bsy\|.
    \end{align*}
\end{assumption}

\begin{proposition}
\label{prop: as cov H S}
    Under Assumptions \ref{ass: lyapunov}-\ref{ass: noise} and \ref{ass: step size rate}-\ref{ass: lipschitz H}, it holds that
    \begin{align*}
        \widehat{\bsH}_t = \bsH + \tilde{o}\Big(t^{-\min\left\{1-\frac{1}{\Bar{p}}, \frac{a}{2}\right\}}\Big) \quad \text{and}\quad \widehat{\bsS}_t = \bsS + \tilde{o}\big(t^{-\frac{a}{2}}\big) \quad \text{a.s.}
    \end{align*}
\end{proposition}


The proof of Proposition \ref{prop: as cov H S} is deferred to Appendix \ref{sec: prop: as cov H S}. The next proposition states that the plug-in estimator $\widehat{\bsH}_t^{-1} \widehat{\bsS}_t \widehat{\bsH}_t^{-1}$ is a reasonable estimator for $\bsH^{-1}\bsS \bsH^{-1}$, with the same rate of convergence as $\widehat{\bsH}_t$. Its proof is deferred to Appendix \ref{sec: prop: as coverge HSH}.

\begin{proposition}
\label{prop: as coverge HSH}
Under Assumptions \ref{ass: lyapunov}-\ref{ass: noise} and \ref{ass: step size rate}-\ref{ass: lipschitz H}, it holds that
\begin{align*}
    \widehat{\bsH}_t^{-1} \widehat{\bsS}_t \widehat{\bsH}_t^{-1} = \bsH^{-1} \bsS \bsH^{-1} + \tilde{o}\Big(t^{-\min\left\{1-\frac{1}{\Bar{p}}, \frac{a}{2}\right\}}\Big) \quad \text{a.s.}
\end{align*}
\end{proposition}


\subsection{Coverage Guarantees}
\label{sec: acs result}

Combining the pieces above, we arrive at the following coverage guarantee for the averaged iterates.

\begin{theorem}[Asymptotic Time-Uniform Coverage]
\label{thm: asymp t1 control}
    Under Assumptions \ref{ass: lyapunov}-\ref{ass: noise} and \ref{ass: step size rate}, for any $\clubsuit \in \{ \text{LIL-UB}, \text{GM}, \text{LIL-EN} \}$ and any $\alpha \in (0,1)$, it holds that
    \begin{align*}
        \lim_{m\rightarrow \infty} \PB \left( \forall t\geq m \colon \Bar{\bsx}_t - \bsx^* \in C_{\widehat{\bsH}_t^{-1}\widehat{\bsS}_t \widehat{\bsH}_t^{-1}}^{\clubsuit}(t) \right) \geq 1-\alpha.
    \end{align*}
\end{theorem}

As opposed to Proposition \ref{prop: general boundary gaussian} where the coverage rate is above $1-\alpha$ for time starting at the beginning, here we only have an ``asymptotic'' coverage rate $1-\alpha$ when $m$ goes to infinity in order to compensate for the errors induced by Gaussian approximation and the plug-in error in covariance estimation. However, we show in experiments below that the coverage rate quickly becomes valid after $m \approx 1000$ iterations, which is just a warm start of the whole training process.

\section{Experiments}
\label{sec: experiment}

In this section we run simulation experiments on two regression problems and illustrate the performance of three confidence sequences.

\paragraph{Linear regression.} For linear model $Y_i = X_i^\top \theta^* + \epsilon_i$, the loss function is $\gL(\theta) = \sum_{i=1}^{n} (Y_i - X_i^\top \theta)^2/2 $, so the SGD update corresponds to \eqref{eq: sa} with $\bsG(\bsx_t,\xi_{t+1}) = (Y_{t+1} - X_{t+1}^\top \theta_t) X_{t+1}$. We assume $\theta^*=(1,\dots,d)^\top$, and generate data as $[X_i]_j\overset{\text{i.i.d.}}{\sim} \operatorname{Unif}(-10, 10)$ and $\epsilon_i \overset{\text{i.i.d.}}{\sim} \gN(0, 16)$.

\paragraph{Logistic regression.} For logit model $Y_i \sim \operatorname{Ber}(\operatorname{logit}(X_i^\top \theta^*))$, the loss function is $\gL(\theta) = -\sum_{i=1}^n \big\{ Y_i\log(1+e^{-X_i^\top \theta}) + (1-Y_i) \log(1+e^{X_i^\top \theta}) \big\}$, so the SGD update corresponds to \eqref{eq: sa} with $\bsG(\bsx_t,\xi_{t+1}) = X_{t+1}Y_{t+1} - X_{t+1}/{(1 + e^{-X_{t+1}^\top \theta_t})}$. We assume $\theta^*=(1,\dots,d)^\top$, and generate data as $[X_i]_j\overset{\text{i.i.d.}}{\sim} \operatorname{Unif}(-0.5, 0.5)$.

We conduct our experiment with Intel(R) Xeon(R) Gold 6132 CPU @ 2.60GHz. For each model, we perform SGD for $1.5 \times 10^5$ iterations and repeat the trajectory for 1,000 times to compute coverage rates. We illustrate in Figures \ref{fig:d=1} and \ref{fig:d=5} the performance of three CSs as well as the traditional fixed-time CI when $d=1$ and $d=5$, respectively. Specifically, we plot the fixed-time coverage rates at each iteration (Column 1), and the time-uniform coverage rates of the whole trajectory before each iteration (Column 2). To mitigate the effect of the approximation errors, we compute such time-uniform coverage rates starting from the 1,000-th iteration. We also plot the CS boundaries (Column 4), and their correspondences where the variance $\bsH^{-1} \bsS \bsH^{-1}$ is replaced by its plug-in estimator (Column 3 and 5). As expected, our CSs are wider than the fixed-width one, and the latter fails to achieve the nominal time-uniform coverage rate. Comparing the boundaries of the CSs, we find the Gaussian mixture boundary (GM) is always the narrowest, achieving around two times of the fixed-time CI length. The boundary based on $\varepsilon$-nets is the widest and often achieves an extremely low coverage rate; this phenomenon is exacerbated in high dimensions. Nevertheless, all three CSs validate the asymptotic coverage guarantee.

\begin{figure}[ht]
    \centering
    \includegraphics[scale=0.25]{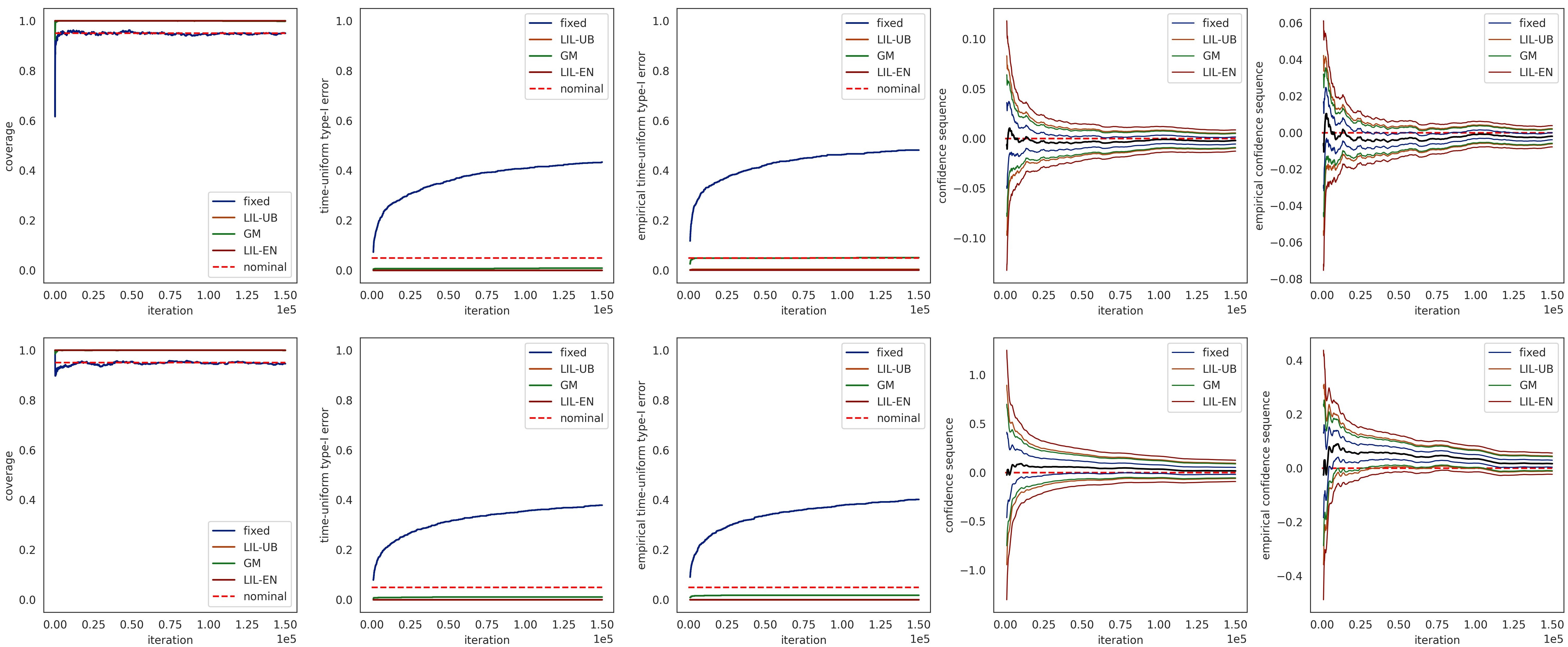}
    \caption{Simulation results for $d=1$. The first row corresponds to linear regression and the second row corresponds to logistic regression. The first column shows fixed-time coverage rates of three CSs as well as the traditional CI (``fixed''); the second/third column displays time-uniform coverage rates without/with the plug-in variance estimator; the fourth/fifth column displays CS boundaries.}
    \label{fig:d=1}
\end{figure}

\begin{figure}[ht]
    \centering
    \includegraphics[scale=0.25]{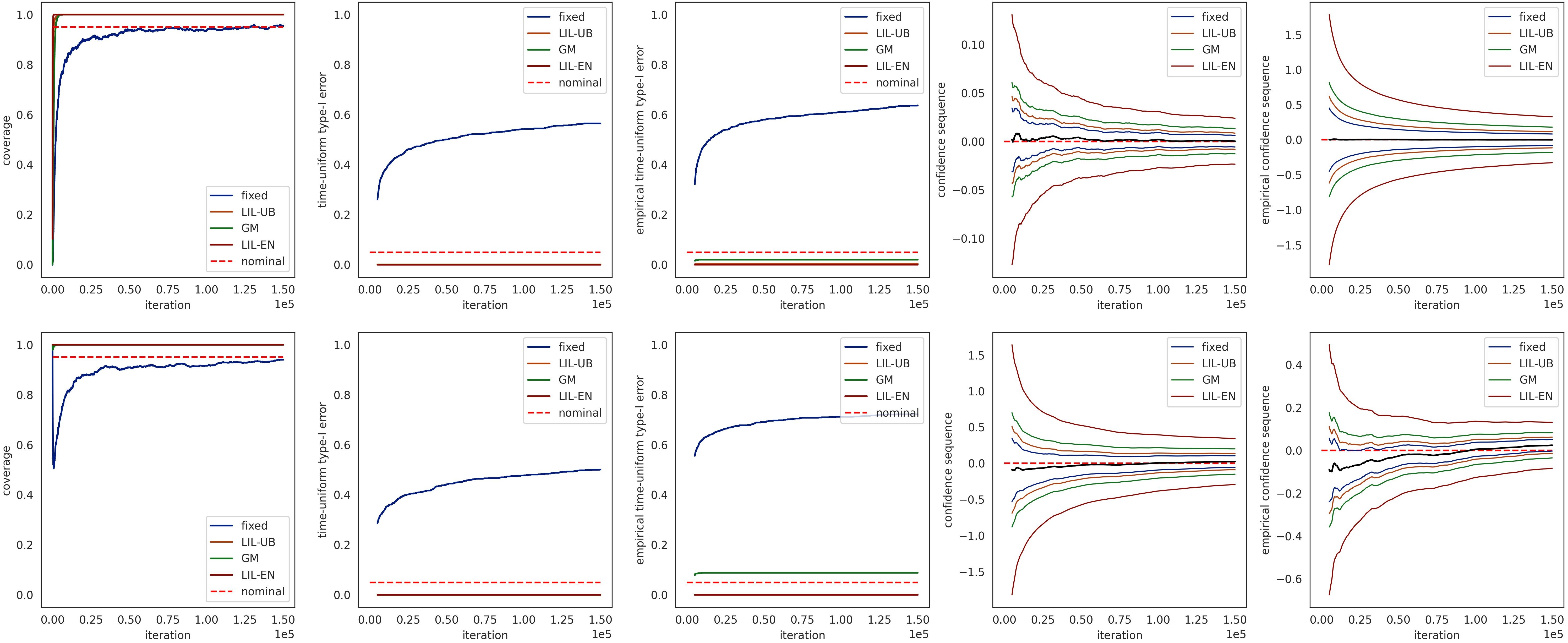}
    \caption{Simulation results for $d=5$, plotted for the first component. Meanings of all plots are the same as in Figure \ref{fig:d=1}.}
    \label{fig:d=5}
\end{figure}

\section{Concluding Remarks}
\label{sec: conclude}

In this work we have leveraged probabilistic tools from anytime-valid inference to establish asymptotic confidence sequences for the averaged iterates from stochastic approximation algorithms. As a by-product, we have derived almost-sure convergence rates of the averaged iterates to a scaled sum of i.i.d.\ Gaussians, and discussed the trade-off between those rates from different causes as well as the best possible rate theoretically. We have proposed three asymptotic confidence sequences and established asymptotic type-I error control guarantees for them. Subsequent numerical experiments validate our results.


There are a few interesting directions for future work. The first is to extend our results to the martingale or Markovian noise settings so that algorithms such as bandit learning \citep{chen2021statistical1,chen2021statistical2,chen2022online}, policy gradient \citep{sutton1999policy} and actor-critic \citep{konda1999actor} can be fit into our framework. The second is to establish asymptotic confidence sequences for the last iterates \citep{pelletier2000asymptotic}.

\clearpage
\bibliographystyle{apalike}
\bibliography{ref.bib}

\clearpage
\appendix

\section{Proofs Related to Almost-Sure Convergence}

\subsection{Proof of Theorem \ref{thm: strong approx}}
\label{sec: thm: strong approx}

The proof is complete by combining Lemma \ref{lem: as clt approx} and Lemma \ref{lem: strong approx} below. The proofs of Lemmas \ref{lem: as clt approx} and \ref{lem: strong approx} can be found in Appendices \ref{sec: lem: as clt approx} and \ref{sec: lem: strong approx}, respectively.

At a high level, we decompose $\Bar{\bsx}_t - \bsx^*$ into several error terms and a well-behaved martingale term as in \eqref{eq: decompose}, where the covariances of increments in this martingale are different but converge to a fixed matrix $\bsH^{-1}\bsS \bsH^{-1}$. Such decomposition is common in the SA literature \citep{polyak1992acceleration,su2023higrad,zhu2021constructing,li2022statistical} and is similarly done in Lemma \ref{lem: as clt approx}, except that we characterize an almost-sure decomposition instead of a usual in-probability one. We also obtain specific convergence rates of all error terms, which are new to the literature to the best of our knowledge.

\begin{lemma}[Approximation of the averaged iterate process]
    \label{lem: as clt approx}
    Let $\epsilon>0$ be an arbitrarily small constant. Under Assumptions \ref{ass: lyapunov}-\ref{ass: noise}, we have
    \begin{align*}
        \Bar{\bsx}_t - \bsx^* = \frac{1}{t} \sum_{j=1}^{t} \bsH^{-1} \bsvareps_{j} + o\left(t^{-1}(\log t)^{\epsilon} + t^{-\frac{a(1+\lambda)}{2}}(\log t)^{1+\epsilon} + t^{\frac{1}{2p} - \frac{2-a}{2}}(\log t)^{\frac{1}{2p}+\epsilon}\right) \quad \text{a.s.}
    \end{align*}
\end{lemma}

To further apply the almost-sure invariance principles \citep{strassen1967almost,philipp1986note} so as to approximate $\Bar{\bsx}_t - \bsx^*$ with a sum of i.i.d. Gaussians, we further calibrate the martingale term so that its increments are i.i.d. with covariance $\bsH^{-1} \bsS \bsH^{-1}$. We derive almost-sure convergence rates of errors due to such calibration as well as application of the almost-sure invariance principles in Lemma \ref{lem: strong approx}.

\begin{lemma}[Strong approximation of the martingale]
\label{lem: strong approx}
    Let $\epsilon>0$ be an arbitrarily small constant. Under Assumptions \ref{ass: lyapunov}-\ref{ass: noise}, there exists a sequence of $d$-dimensional i.i.d. centered Gaussian random vectors $\{\bsG_t\}_{t\geq 1}$ with covariance $\bsH^{-1}\bsS \bsH^{-1}$ on the same (potentially enriched) probability space, such that
    \begin{align*}
        \frac{1}{t} \sum_{j=1}^{t} \bsH^{-1} \bsvareps_{j} = \frac{1}{t} \sum_{j=1}^t \bsG_j + o\left( t^{-\frac{1+a}{2}}(\log t)^{\frac{1}{2}+\epsilon} + t^{-\frac{1}{2} - \frac{p-1}{50dp}} (\log t)^{\frac{1}{50dp} + \epsilon} \right) \quad \text{a.s.}
    \end{align*}
    If in addition $d=1$, then the second rate on the right-hand side can be improved to $o\big(t^{-\frac{1}{2} - \frac{p-1}{4p}}(\log t)^{1+\frac{1}{4p} + \epsilon}\big)$.
\end{lemma}

\subsection{Proof of Proposition \ref{prop: as cov H S}}
\label{sec: prop: as cov H S}
\paragraph{Convergence rate for $\widehat{\bsH}_t$.} We decompose $\widehat{\bsH}_t - \bsH$ into the following terms:
\begin{align*}
    \widehat{\bsH}_t - \bsH &= \frac{1}{t} \sum_{j=0}^{t-1} \Bar{\bsH}(\bsx_j, \xi_{j+1}) - \bsH \\
    &=  \underbrace{\frac{1}{t} \sum_{j=0}^{t-1} \left( \Bar{\bsH}(\bsx^*, \xi_{j+1}) - \bsH \right)}_{=:\gT_{1,t}} + \underbrace{\frac{1}{t} \sum_{j=0}^{t-1} \left( \Bar{\bsH}(\bsx_j, \xi_{j+1}) - \Bar{\bsH}(\bsx^*, \xi_{j+1}) \right)}_{=:\gT_{2,t}}.
\end{align*}

The first term is a mean of i.i.d. random variables, so by Theorems 2.4.1 and 2.5.12 of \cite{durrett2019probability}, if $1\leq \Bar{p} < 2$ then $\|\gT_{1,t}\| = o\big(t^{\frac{1}{\Bar{p}}-1}\big)$ almost surely; and by Theorem 2.5.11 of \cite{durrett2019probability}, if $\Bar{p} \geq 2$ then $\|\gT_{1,t}\| = o\big(t^{-\frac{1}{2}}(\log t)^{\frac{1}{2} + \epsilon}\big)$ almost surely.

For the second term, note that for any $\epsilon>0$,
\begin{align*}
    \sum_{t=2}^\infty \frac{\E \|\Bar{\bsH}(\bsx_t, \xi_{t+1}) - \Bar{\bsH}(\bsx^*, \xi_{t+1})\|}{t^{\frac{2-a}{2}}(\log t)^{1+\epsilon}} &{\leq} \sum_{t=2}^\infty \frac{\E\Big\{ \big( \E(\|\Bar{\bsH}(\bsx_t, \xi_{t+1}) - \Bar{\bsH}(\bsx^*, \xi_{t+1})\|^2\mid \gF_t)\big)^{1/2}\Big\}}{t^{\frac{2-a}{2}}(\log t)^{1+\epsilon}} \\
    &\overset{(a)}{\leq} \sum_{t=2}^\infty \frac{L_{\Bar{H}}\E\|\bsx_t - \bsx^*\|}{t^{\frac{2-a}{2}}(\log t)^{1+\epsilon}} \leq \sum_{t=2}^\infty \frac{L_{\Bar{H}}\big(\E\|\bsx_t - \bsx^*\|^2\big)^{1/2}}{t^{\frac{2-a}{2}}(\log t)^{1+\epsilon}} \\
    &\overset{(b)}{\lesssim} \sum_{t=2}^{\infty} \frac{1}{t(\log t)^{1+\epsilon}} < \infty,
\end{align*}
where (a) uses Assumption \ref{ass: lipschitz H}, and (b) uses $\E\|\bsx_t - \bsx^*\|^2 \lesssim \eta_t \lesssim t^{-a}$ from Lemma \ref{lem: as l2 conv}. Hence we have with probability one,
\begin{align*}
    \sum_{t=2}^\infty \frac{\|\Bar{\bsH}(\bsx_t, \xi_{t+1}) - \Bar{\bsH}(\bsx^*, \xi_{t+1})\|}{t^{\frac{2-a}{2}}(\log t)^{1+\epsilon}} < \infty.
\end{align*}
By Kronecker's lemma,
\begin{align*}
    \frac{1}{t^{\frac{2-a}{2}}(\log t)^{1+\epsilon}} \sum_{j=0}^{t-1} \|\Bar{\bsH}(\bsx_j, \xi_{j+1}) - \Bar{\bsH}(\bsx^*, \xi_{j+1})\| \rightarrow 0 \quad \text{a.s.},
\end{align*}
and therefore $\|\gT_{2,t}\| = o\big(t^{-1}\cdot t^{\frac{2-a}{2}}(\log t)^{1+\epsilon}\big) = o\big(t^{-\frac{a}{2}}(\log t)^{1+\epsilon}\big)$.

Combining the above two terms yields the final almost-sure convergence rate of $\widehat{\bsH}_t$:
\begin{align*}
    \widehat{\bsH}_t - \bsH = o\left(t^{\frac{1}{\Bar{p}}-1}\indicator(1\leq \Bar{p} < 2) + t^{-\frac{1}{2}} (\log t)^{\frac{1}{2} + \epsilon} \indicator(\Bar{p}\geq 2) + t^{-\frac{a}{2}}(\log t)^{1+\epsilon} \right) = \Tilde{o}\left(t^{-\min\left\{1-\frac{1}{\Bar{p}}, \frac{a}{2}\right\}}\right).
\end{align*}

\paragraph{Convergence rate for $\widehat{\bsS}_t$.} We decompose
\begin{align*}
    \bsG(\bsx_j,\xi_{j+1}) = \underbrace{\bsG(\bsx^*, \xi_{j+1})}_{=: \bsu_{j+1}} + \underbrace{(\bsG(\bsx_{j},\xi_{j+1}) - \bsG(\bsx^*,\xi_{j+1}))}_{=: \bsv_{j+1}},
\end{align*}
so by Assumptions \ref{ass: noise} \ref{ass: noise 3} and \ref{ass: lipschitz}, we have $\E \|\bsu_{j+1}\|^{2p} \leq C_\varepsilon$ and $\E\|\bsv_{j+1}\|^2 \leq L_G^2 \E\|\bsx_{j} - \bsx^*\|^2$. We can write $\widehat{\bsS}_t - \bsS$ as
\begin{align*}
    \widehat{\bsS}_t - \bsS = \underbrace{\frac{1}{t}\sum_{j=1}^t (\bsu_j\bsu_j^\top - \bsS)}_{=:\gT_{1,t}} + \underbrace{\frac{1}{t}\sum_{j=1}^t \bsu_j \bsv_j^\top + \frac{1}{t}\sum_{j=1}^t \bsv_j \bsu_j^\top}_{=: \gT_{2,t}} + \underbrace{\frac{1}{t}\sum_{j=1}^t \bsv_j \bsv_j^\top}_{=: \gT_{3,t}}.
\end{align*}

The first term is again a mean of i.i.d. random variables with finite $p$-th moment, since
\begin{align*}
    \E\|\bsu_j \bsu_j^\top - \bsS\|^p \lesssim \E\|\bsu_j \bsu_j^\top\|^p = \E\|\bsu_j\|^{2p} < C_\varepsilon.
\end{align*}
In addition, by Assumption \ref{ass: step size rate}, we must have $p>2$. Similar to the previous analysis for $\widehat{\bsH}_t$, it is easy to derive $\|\gT_{1,t}\| = o\big( t^{-\frac{1}{2}}(\log t)^{\frac{1}{2}+\epsilon} \big)$.

For the third term, note that for any $\epsilon>0$,
\begin{align*}
    \sum_{t=2}^\infty \frac{\E\|\bsv_t\bsv_t^\top\|}{t^{1-a}(\log t)^{1+\epsilon}} &\overset{(a)}{\leq} \sum_{t=2}^\infty \frac{L_G^2 \E\|\bsx_{t-1} - \bsx^*\|^2}{t^{1-a}(\log t)^{1+\epsilon}} \overset{(b)}{\lesssim} \sum_{t=2}^\infty \frac{1}{t(\log t)^{1+\epsilon}} < \infty,
\end{align*}
where (a) uses $\E\|\bsv_t \bsv_t^\top \| = \E \|\bsv_t\|^2$ and (b) uses $\E\|\bsx_{t-1} - \bsx^*\|^2 \lesssim \eta_t \lesssim t^{-a}$ from Lemma \ref{lem: as l2 conv}. Thus with probability one,
\begin{align*}
    \sum_{t=2}^\infty \frac{\|\bsv_t\bsv_t^\top\|}{t^{1-a}(\log t)^{1+\epsilon}} < \infty,
\end{align*}
and by Kronecker's lemma, 
\begin{align*}
    \frac{1}{t^{1-a}(\log t)^{1+\epsilon}} \sum_{j=1}^t \|\bsv_j \bsv_j^\top \| \rightarrow 0 \quad \text{a.s.},
\end{align*}
which immediately implies $\|\gT_{3,t}\| = o\big(t^{-1}\cdot t^{1-a}(\log t)^{1+\epsilon}\big) = o\big( t^{-a}(\log t)^{1+\epsilon} \big)$.

For the second term, by Cauchy's inequality and the boundedness of $\E\|\bsu_t\|^2$ we have $\E\|\bsu_t \bsv_t^\top\| \leq \big(\E\|\bsu_t\|^2 \E\|\bsv_t\|^2 \big)^{1/2} \lesssim \big(\E\|\bsv_t\|^2\big)^{1/2}$. Thus, by a similar analysis for $\gT_{3,t}$ above, we can obtain $\|\gT_{2,t}\| = o\big( t^{-\frac{a}{2}}(\log t)^{1+\epsilon} \big)$.

Summing up $\gT_{1,t}, \gT_{2,t}, \gT_{3,t}$ yields the final rate of convergence for $\widehat{\bsS}_t$:
\begin{align*}
    \widehat{\bsS}_t - \bsS = o\left( t^{-\frac{1}{2}}(\log t)^{\frac{1}{2}+\epsilon} + t^{-\frac{a}{2}}(\log t)^{1+\epsilon} + t^{-\frac{1}{2}}(\log t)^{\frac{1}{2}+\epsilon} \right) = \Tilde{o}\left( t^{-\frac{a}{2}} \right).
\end{align*}


\subsection{Proof of Proposition \ref{prop: as coverge HSH}}
\label{sec: prop: as coverge HSH}
Note that by Proposition \ref{prop: as cov H S}, $\widehat{\bsH}_t \xrightarrow[]{a.s.} \bsH$ and $\widehat{\bsS}_t \xrightarrow[]{a.s.} \bsS$, and consequently, $\widehat{\bsH}_t$ is invertible for sufficiently large $t>0$ and $\widehat{\bsH}_t^{-1} \xrightarrow[]{a.s.} \bsH^{-1}$. This implies both $\|\widehat{\bsH}_t^{-1}\|$ and $\|\widehat{\bsS}_t\|$ are bounded with probability one. 
Therefore, we have
\begin{align*}
    \|\widehat{\bsH}^{-1}_t - \bsH^{-1}\| &\leq \|\bsH^{-1}\|\|\widehat{\bsH}_t - \bsH\|\|\widehat{\bsH}_t^{-1}\| = \Tilde{o}\left(t^{-\min\left\{1-\frac{1}{\Bar{p}}, \frac{a}{2}\right\}}\right),
\end{align*}
and
\begin{align*}
    &\quad\ \left\| \widehat{\bsH}^{-1}_t \widehat{\bsS}_t \widehat{\bsH}^{-1}_t - \bsH^{-1} \bsS \bsH^{-1} \right\| \\
    &= \left\| \bsH^{-1}\bsS (\widehat{\bsH}^{-1}_t - \bsH^{-1}) + \bsH^{-1}(\widehat{\bsS}_t - \bsS) \widehat{\bsH}^{-1}_t + (\widehat{\bsH}^{-1}_t - \bsH^{-1}) \widehat{\bsS}_t \widehat{\bsH}^{-1}_t \right\| \\
    &\leq \|\bsH^{-1}\|\|\bsS\| \|\widehat{\bsH}^{-1}_t - \bsH^{-1}\| + \|\bsH^{-1}\|\|\widehat{\bsS}_t - \bsS\| \|\widehat{\bsH}^{-1}_t\| + \|\widehat{\bsH}^{-1}_t - \bsH^{-1}\| \|\widehat{\bsS}_t\| \|\widehat{\bsH}^{-1}_t\| \\
    &= \Tilde{o}\left(t^{-\min\left\{1-\frac{1}{\Bar{p}}, \frac{a}{2}\right\}}\right) + \Tilde{o}\left( t^{-\frac{a}{2}} \right) + \Tilde{o}\left(t^{-\min\left\{1-\frac{1}{\Bar{p}}, \frac{a}{2}\right\}}\right) = \Tilde{o}\left(t^{-\min\left\{1-\frac{1}{\Bar{p}}, \frac{a}{2}\right\}}\right),
\end{align*}
which concludes the proof.
\section{Proofs Related to Confidence Sequences}
\label{app: proof cs}

\subsection{Proof of Proposition \ref{prop: general boundary gaussian}}

The proof of complete by separately ensuring time-uniform coverage for three confidence sequences, as accomplished in Proposition \ref{prop: lil bound}, Proposition \ref{prop: gauss mix} and Corollary \ref{coro: gauss mix}, and Proposition \ref{prop: lil en} below, respectively.

\subsection{LIL Boundary with Union Bounds}
\label{app: lil ub}

\begin{proposition}[LIL-UB boundary]
\label{prop: lil bound}
    Let $\bsG_1,\bsG_2,\dots$ be i.i.d. Gaussian random vectors with mean zero and covariance $\bsH^{-1}\bsS \bsH^{-1}$, and let $\bsM_t := \frac{1}{t}\sum_{j=1}^t \bsG_j$. Then for any $\alpha \in (0,1)$, it holds that
    \begin{align*}
        \PB \left( \exists t\geq 1\colon \left\| \big(\bsH^{-1}\bsS\bsH^{-1}\big)^{-1/2} \bsM_t \right\|_\infty \geq 1.7 \sqrt{\frac{ \log\log(2t) + 0.72 \log(10.4d/\alpha) }{t}} \right) \leq \alpha.
    \end{align*}
\end{proposition}

\begin{proof}[Proof of Proposition \ref{prop: lil bound}]
We first state a lemma on the LIL boundary for Gaussians in the one-dimensional case.

\begin{lemma}
\label{lem: lil}
    Let $X_1,X_2,\dots$ be i.i.d. standard Gaussians, and let $S_t := \sum_{j=1}^t X_j$. Then for any $\alpha\in (0,1)$,
    \begin{align*}
        \PB \left( \exists t\geq 1\colon |S_t| \geq 1.7 \sqrt{t \left( \log\log(2t) + 0.72 \log \left(\frac{10.4}{\alpha}\right) \right)} \right) \leq \alpha.
    \end{align*}
\end{lemma}

The proof of Lemma \ref{lem: lil} is trivial by noticing (3.4) of \cite{howard2021time} and doubling the constant 5.2 to accommodate the two-sided version of the concentration inequality.

Now we let $\widetilde{\bsG}_t := \big(\bsH^{-1}\bsS\bsH^{-1}\big)^{-1/2} \bsG_t$ and $\widetilde{\bsM}_t := \frac{1}{t}\sum_{j=1}^t \widetilde{\bsG}_j$. Obviously, $\widetilde{\bsG}_t \overset{\text{i.i.d.}}{\sim} \gN(\0, \bsI)$, and thus we can use a union bound to ensure that with high probability, each coordinate of $\widetilde{\bsM}_t$ is within the confidence sequence:
\begin{align*}
    &\quad \ \PB \left( \exists t\geq 1\colon \left\| \big(\bsH^{-1}\bsS\bsH^{-1}\big)^{-1/2} \bsM_t \right\|_\infty \geq 1.7 \sqrt{\frac{ \log\log(2t) + 0.72 \log(10.4d/\alpha) }{t}} \right) \\
    &= \PB \left( \exists t\geq 1\colon \left\| \widetilde{\bsM}_t \right\|_\infty \geq 1.7 \sqrt{\frac{ \log\log(2t) + 0.72 \log(10.4d/\alpha) }{t}} \right) \\
    &=  \sum_{i=1}^{d} \PB \left( \exists t\geq 1\colon \left| [\widetilde{\bsM}_t]_{i} \right| \geq 1.7 \sqrt{\frac{ \log\log(2t) + 0.72 \log(10.4d/\alpha) }{t}} \right) \\
    &\leq \sum_{i=1}^{d} \PB \left( \exists t\geq 1\colon \left| \sum_{j=1}^t[\widetilde{\bsG}_j]_{i}\right| \geq 1.7 \sqrt{t \left( \log\log(2t) + 0.72 \log \left(\frac{10.4d}{\alpha}\right) \right)} \right) \\
    &\leq d \cdot \frac{\alpha}{d} = \alpha,
\end{align*}
which concludes the proof.
\end{proof}

Note that Lemma \ref{lem: lil} is a special case of Theorem 1 of \cite{howard2021time} by setting the boundary shape function $h(k) = (k+1)^s \zeta(s)$ ($\zeta(s)$ is the Riemann zeta function) with $s=1.4$, the geometric spacing of intrinsic time $\eta=2$, and the starting intrinsic time of a nontrivial boundary $m=1$ (see Section 3.1 of \cite{howard2021time} for detailed discussion). These hyperparameter choices will be reused in Appendix \ref{app: lil en} to facilitate fair comparison.

\subsection{Gaussian Mixture Boundary}
\label{app: gauss mix}

\begin{proposition}[GM boundary]
\label{prop: gauss mix}
    Let $\bsG_1,\bsG_2,\dots$ be i.i.d. Gaussian random vectors with mean zero and covariance $\bsH^{-1}\bsS \bsH^{-1}$, and let $\bsM_t := \frac{1}{t}\sum_{j=1}^t \bsG_j$. Then for any positive definite matrix $\bsSigma \in \SB_+^{d}$ and any $\alpha \in (0,1)$, it holds that
    \begin{align*}
        \PB\left( \exists t\geq 1\colon \|\bsM_t\|_{\{\bsH^{-1} \bsS \bsH^{-1} + \bsSigma^{-1}/t\}^{-1}} \geq \sqrt{\frac{\log\det\{t\bsH^{-1} \bsS \bsH^{-1} \bsSigma + \bsI \} + 2\log(1/\alpha)}{t} } \right) \leq \alpha.
    \end{align*}
\end{proposition}

\begin{proof}[Proof of Proposition \ref{prop: gauss mix}]
Note that for any martingale $M_t(\lambda)$, we have that $\int M_t(\lambda)\rd F(\lambda)$ is also a martingale where $F$ is any probability distribution \citep{howard2020time,howard2021time}. We use the density of a $d$-dimensional Gaussian $\gN(\0,\bsSigma)$, to mix the following exponential martingale:
\begin{align*}
    M_t(\lambda) := \exp\left( \sum_{j=1}^t \langle \bslambda, \bsG_j \rangle - \frac{t}{2} \langle \bslambda, \bsH^{-1}\bsS \bsH^{-1} \bslambda \rangle \right).
\end{align*}
Let $\widetilde{\bsSigma}_t^{-1} := t \bsH^{-1}\bsS \bsH^{-1} + \bsSigma^{-1}$. Direct calculation yields
\begin{align*}
    \int M_t(\lambda) \rd F(\lambda) &= \frac{1}{(2\pi)^{d/2}|\det(\bsSigma)|^{1/2}}\int \exp\left( t \bslambda^\top \bsM_t - \frac{t}{2} \bslambda^\top \bsH^{-1}\bsS \bsH^{-1} \bslambda - \frac{1}{2}\bslambda^\top \bsSigma^{-1} \bslambda \right) \rd \bslambda \\
    &= \frac{1}{(2\pi)^{d/2}|\det(\bsSigma)|^{1/2}}\int \exp\left( t \bslambda^\top \bsM_t - \frac{1}{2} \bslambda^\top \widetilde{\bsSigma}_t^{-1} \bslambda \right) \rd \bslambda \\
    &= \frac{1}{(2\pi)^{d/2}|\det(\bsSigma)|^{1/2}}\int \exp\left( - \frac{1}{2} \left(\bslambda - t\widetilde{\bsSigma}_t \bsM_t \right)^\top \widetilde{\bsSigma}_t^{-1} \left(\bslambda - t\widetilde{\bsSigma}_t \bsM_t \right) + \frac{t^2}{2} \bsM_t^\top \widetilde{\bsSigma}_t \bsM_t \right) \rd \bslambda \\
    &= \frac{(2\pi)^{d/2}|\det(\widetilde{\bsSigma}_t)|^{1/2}}{(2\pi)^{d/2}|\det(\bsSigma)|^{1/2}} \exp \left( \frac{t^2}{2} \bsM_t^\top \widetilde{\bsSigma}_t \bsM_t \right) \\
    &= \frac{\exp\left( (t^2/2) \cdot \bsM_t^\top \{ t\bsH^{-1} \bsS \bsH^{-1} + \bsSigma^{-1} \}^{-1} \bsM_t \right)}{|\det(\bsSigma)|^{1/2}|\det\{ t\bsH^{-1} \bsS \bsH^{-1} + \bsSigma^{-1} \}|^{1/2}}.
\end{align*}
Since the above term in the last line is a nonnegative martingale, by Ville's inequality (which we present as Lemma \ref{lem: ville ineq}), we have
\begin{align}
    &\PB \left( \exists t\geq 1\colon \frac{\exp\left( (t^2/2) \cdot \bsM_t^\top \{ t\bsH^{-1} \bsS \bsH^{-1} + \bsSigma^{-1}\}^{-1} \bsM_t \right)}{|\det(\bsSigma)|^{1/2}|\det\{t\bsH^{-1} \bsS \bsH^{-1} + \bsSigma^{-1} \}|^{1/2}} \geq \frac{1}{\alpha} \right) \leq \alpha \quad \Longleftrightarrow \notag \\
    &\PB\left( \exists t\geq 1\colon \|\bsM_t\|_{\{\bsH^{-1} \bsS \bsH^{-1} + \bsSigma^{-1}/t\}^{-1}} \geq \sqrt{\frac{\log\det\{t\bsH^{-1} \bsS \bsH^{-1} \bsSigma + \bsI \} + 2\log(1/\alpha)}{t} } \right) \leq \alpha, \label{eq: gmm proof cs}
\end{align}
which concludes the proof.
\end{proof}

The boundary in Proposition \ref{prop: gauss mix} contains a hyperparameter $\bsSigma$. In Corollary \ref{coro: gauss mix} below we optimize over $\bsSigma$ to obtain a neater result.

\begin{corollary}[GM boundary]
\label{coro: gauss mix}
    Let $W_{-1}(x)$ be the lower branch of the Lambert $W$ function, i.e., the most negative real-valued solution in $z$ to $z e^z = x$, and let $\lambda^* := - W_{-1}( -\alpha^2/e ) - 1$.
    Then for fixed $t_0>0$, choosing $\bsSigma = (\lambda^*/t_0) (\bsH^{-1} \bsS \bsH^{-1})^{-1}$ in Proposition \ref{prop: gauss mix} minimizes the volume of its confidence region at time $t_0$, and the probabilistic guarantee becomes
    \begin{align*}
        \PB\left( \exists t\geq 1\colon \|\bsM_t\|_{(\bsH^{-1} \bsS \bsH^{-1})^{-1}} \geq \sqrt{\frac{\{1 + t_0/(\lambda^* t)\}\{d\log(1 + \lambda^*t/t_0) + 2\log(1/\alpha) \} }{t} } \right) \leq \alpha.
    \end{align*}
\end{corollary}

\begin{proof}[Proof of Corollary \ref{coro: gauss mix}]
In this proof, we optimize over $\bsSigma$ in Proposition \ref{prop: gauss mix} to minimize the volume of the confidence region at a fixed time $t_0>0$. Note that the confidence region at time $t_0$ implied by \eqref{eq: gmm proof cs} is a $d$-dimensional ellipsoid, so by Lemma \ref{lem: ellipsoid}, its volume is 
\begin{align}
    \frac{2}{d}\frac{\pi^{d/2}}{\Gamma(d/2)} \frac{|\det(t_0\bsH^{-1}\bsS\bsH^{-1} + \bsSigma^{-1})|^{1/2} \{ \log \det (t_0\bsH^{-1}\bsS\bsH^{-1} \bsSigma + \bsI) + 2\log(1/\alpha) \}^{d/2} }{t_0^{d}}. \label{eq: volume}
\end{align}
For notational simplicity, denote $\bsA := t_0\bsH^{-1}\bsS\bsH^{-1} \bsSigma$. The numerator of the last term in \eqref{eq: volume} can be written as
\begin{align}
    &\quad\ |\det((\bsA + \bsI) \bsSigma^{-1})|^{1/2} \{ \log \det (\bsA + \bsI) + 2\log(1/\alpha) \}^{d/2} \notag \\
    &= |\det((\bsA + \bsI) \bsA^{-1}\cdot t_0\bsH^{-1}\bsS\bsH^{-1})|^{1/2} \{ \log \det (\bsA + \bsI) + 2\log(1/\alpha) \}^{d/2} \notag \\
    &= |\det(t_0\bsH^{-1}\bsS\bsH^{-1})|^{1/2} \cdot \frac{|\det(\bsA + \bsI)|^{1/2}}{|\det \bsA|^{1/2}} \{ \log \det (\bsA + \bsI) + 2\log(1/\alpha) \}^{d/2}. \label{eq: volume wrt A}
\end{align}
Let $\bsU \bsLambda \bsU^\top$ be the eigenvalue decomposition of $\bsA$ where $\bsLambda = \diag \{\lambda_1,\dots, \lambda_{d}\}$. Then the product of the last two terms in \eqref{eq: volume wrt A} can be written as
\begin{align}
    \frac{\sqrt{\prod_{i=1}^{d} (1+\lambda_i)}}{\sqrt{\prod_{i=1}^{d} \lambda_i}} \left\{ \log \prod_{i=1}^{d}(1+\lambda_i) + 2\log(1/\alpha) \right\}^{d/2}. \label{eq: volume wrt lambda_i}
\end{align}
We claim that when $\prod_{i=1}^{d} (1+\lambda_i)$ is fixed, $\prod_{i=1}^{d} \lambda_i$ attains its maximum at $\lambda_1 = \dots = \lambda_{d}$. To see this, note that for any $i\neq j$, holding $(1+\lambda_i)(1+\lambda_j) = c_{ij}$ constant, we have
\begin{align*}
    c_{ij} = 1 + \lambda_i + \lambda_j + \lambda_i\lambda_j \geq 1 + 2\sqrt{\lambda_i \lambda_j} + \lambda_i\lambda_j,
\end{align*}
and the equality holds if and only if $\lambda_i = \lambda_j$. So when $\lambda_i\neq \lambda_j$, we can always modify them to equal $\lambda_i' = \lambda_j' = \sqrt{c_{ij}} - 1$ to improve the value of $\prod_{i=1}^{d} \lambda_i$. 

Setting $\lambda_i$'s to be equal, \eqref{eq: volume wrt lambda_i} further simplifies to
\begin{align}
    \left( \frac{1+\lambda}{\lambda} \right)^{d/2} \left\{ d\log(1+\lambda) + 2\log(1/\alpha) \right\}^{d/2}. \label{eq: volume wrt lambda}
\end{align}
The minimum of \eqref{eq: volume wrt lambda} is achieved at $\lambda^* = - W_{-1}( -\alpha^2/e ) - 1$, where $W_{-1}(x)$ denotes the lower branch of the Lambert $W$ function, i.e., the most negative real-valued solution in $z$ to $z e^z = x$.

Now we turn back to the construction of $\bsSigma$. The optimal choice for $\bsLambda$ is $\lambda^* \bsI$ as discussed above. In addition, the choice of $\bsU$ does not affect the value in \eqref{eq: volume wrt A}, so without loss of generality we set it as $\bsI$. In this case, we obtain
\begin{align*}
    \bsSigma = \frac{1}{t_0} (\bsH^{-1} \bsS \bsH^{-1})^{-1} \bsA = \frac{\lambda^*}{t_0} (\bsH^{-1} \bsS \bsH^{-1})^{-1} = \frac{- W_{-1}( -\alpha^2/e ) - 1}{t_0} (\bsH^{-1} \bsS \bsH^{-1})^{-1},
\end{align*}
which concludes the proof.
\end{proof}

\subsection{LIL Boundary with $\varepsilon$-Nets}
\label{app: lil en}

\begin{proposition}[LIL-EN boundary]
\label{prop: lil en}
    Let $\bsG_1,\bsG_2,\dots$ be i.i.d. Gaussian random vectors with mean zero and covariance $\bsH^{-1}\bsS \bsH^{-1}$, and let $\bsM_t := \frac{1}{t}\sum_{j=1}^t \bsG_j$. Then for any $\epsilon\in (0,1)$ and any $\alpha \in (0,1)$, it holds that
    \begin{align*}
        \PB \left( \exists t\geq 1\colon \left\| \big(\bsH^{-1}\bsS\bsH^{-1}\big)^{-1/2} \bsM_t \right\|_2 \geq \frac{2}{1-\epsilon} \sqrt{\frac{1.4\log\log(2\Bar{\kappa}t) + \Bar{L}_{d,\epsilon,\alpha}}{t}} \right) \leq \alpha,
    \end{align*}
    where $\Bar{\kappa}$ is the conditional number of $\bsH^{-1}\bsS \bsH^{-1}$, and
    \begin{align*}
        \Bar{L}_{d,\epsilon,\alpha} = \log\left( \frac{5.2C_d}{\alpha}\right) + (d-1)\log\left(\frac{3\sqrt{\Bar{\kappa}}}{\epsilon}\right),\quad C_d = \frac{d 2^d \Gamma(\frac{d+1}{2})}{\pi^{\frac{d-1}{2}}}.
    \end{align*}
\end{proposition}

\begin{proof}[Proof of Proposition \ref{prop: lil en}]
    The proof follows the same spirit of Theorem 4.1 of \cite{whitehouse2023timeuniform} with slight modification. We start from a univariate boundary result.

    \begin{lemma}
    \label{lem: lil en uni}
        Let $G_1,G_2,\dots$ be i.i.d. Gaussian random variables with mean zero and covariance $V$, and let $S_t := \sum_{j=1}^t G_j, V_t := tV$. Let $\eta>1, \rho>0, s>1$ be some fixed constants. Then for any $\alpha \in (0,1)$, it holds that
        \begin{align*}
            \PB \left( \exists t\geq 1\colon S_t \geq \sqrt{2\eta (V_t \lor \rho) \left\{ s\log\log\left( \frac{\eta (V_t \lor \rho)}{\rho} \right) + \log \left( \frac{\zeta(s)}{\alpha \log^s \eta} \right) \right\}} \right) \leq \alpha.
        \end{align*}
    \end{lemma}
    Lemma \ref{lem: lil en uni} follows from Theorem 3.1 of \cite{whitehouse2023timeuniform} by letting $\psi(x) = x^2/2$ and $h(k) = (k+1)^s \zeta(s)$ where $\zeta(s)$ is the Riemann zeta function.

    Next, we use the $\varepsilon$-net arguments to obtain a time-uniform boundary for the normalized multivariate process.

    \begin{lemma}
    \label{lem: lil en multi}
        Let $\bsG_1,\bsG_2,\dots$ be i.i.d. Gaussian random vectors with mean zero and covariance $\bsV$, and let $\bsS_t := \sum_{j=1}^t \bsG_j, \bsV_t := t\bsV$. Let $\eta>1, s>1, \epsilon\in (0,1)$ be some fixed constants, $\Bar{\kappa}:=\kappa(\bsV)$ be the conditional number of $\bsV$, and $N_0 := N(\SB^{d-1}, \epsilon/\sqrt{\Bar{\kappa}}, \|\cdot\|_2)$ be the covering number of the unit sphere. Then for any $\alpha \in (0,1)$, it holds that
        \begin{align*}
            \PB \left( \exists t\geq 1\colon \left\| \bsV_t^{-1/2} \bsS_t \right\|_2 \geq \frac{1}{1-\epsilon} \sqrt{2\eta \left\{ s\log\log(\eta \Bar{\kappa} t) + \log\left( \frac{N_0 \zeta(s)}{\alpha \log^s\eta} \right) \right\}} \right) \leq \alpha.
        \end{align*}
    \end{lemma}
    \begin{proof}[Proof of Lemma \ref{lem: lil en multi}]
        Let $K_0$ be the corresponding minimal $\frac{\epsilon}{\sqrt{\Bar{\kappa}}}$-cover of the unit sphere $\SB^{d-1}$ with $|K_0|=N_0$. For any fixed $\nu \in \SB^{d-1}$, by Lemma \ref{lem: lil en uni} with $S_t = \langle \nu, \bsS_t \rangle$, $V_t = \langle \nu, \bsV_t \nu \rangle$ and $\rho = \gamma_{\min}(\bsV)$, we have
        \begin{align*}
            \PB \left( \exists t\geq 1\colon \langle \nu, \bsS_t \rangle \geq \sqrt{2\eta \langle \nu, \bsV_t \nu \rangle \left\{ s\log\log\left( \frac{\eta \langle \nu, \bsV_t \nu \rangle}{\gamma_{\min}(\bsV)} \right) + \log \left( \frac{N_0 \zeta(s)}{\alpha \log^s\eta} \right) \right\}} \right) \leq \frac{\alpha}{N_0}.
        \end{align*}
        A straightforward union bound over all $\nu\in K_0$ yields
        \begin{align}
            & \PB \left( \exists t\geq 1, \nu \in K_0 \colon \langle \nu, \bsS_t \rangle \geq \sqrt{2\eta \langle \nu, \bsV_t \nu \rangle \left\{ s\log\log(\eta \Bar{\kappa} t) + \log \left( \frac{N_0 \zeta(s)}{\alpha \log^s \eta} \right) \right\}} \right) \notag \\
            \leq\ & \PB \left( \exists t\geq 1, \nu \in K_0 \colon \langle \nu, \bsS_t \rangle \geq \sqrt{2\eta \langle \nu, \bsV_t \nu \rangle \left\{ s\log\log \left( \frac{\eta \langle \nu, \bsV_t\nu \rangle }{\gamma_{\min}(\bsV)} \right) + \log \left( \frac{N_0 \zeta(s)}{\alpha \log^s \eta} \right) \right\}} \right) \notag \\
            \leq\ & \alpha. \label{eq: lil en union bd}
        \end{align}
        Now we bound $\left\| \bsV_t^{-1/2} \bsS_t \right\|_2$. Let $K_0(t):= \left\{\bsV_t^{1/2}\nu / \|\bsV_t^{1/2}\nu \|_2 \colon \nu \in K_0 \right\} \subset \SB^{d-1}$ and $\pi_t \colon \SB^{d-1} \rightarrow K_0(t)$ be the $L^2$ projection onto the finite set $K_0(t)$. For any $t\geq 1$,
        \begin{align*}
            \left\| \bsV_t^{-1/2} \bsS_t \right\|_2 &= \sup_{\omega \in \SB^{d-1}} \left\langle \omega, \bsV_t^{-1/2} \bsS_t \right\rangle = \sup_{\omega \in \SB^{d-1}} \left\{ \left\langle \omega - \pi_t(\omega) , \bsV_t^{-1/2} \bsS_t \right\rangle + \left\langle \pi_t(\omega), \bsV_t^{-1/2} \bsS_t \right\rangle \right\} \\
            &\leq \sup_{\omega \in \SB^{d-1}} \left\| \omega - \pi_t(\omega) \right\|_2 \left\| \bsV_t^{-1/2} \bsS_t \right\|_2 + \sup_{\omega \in K_0(t)} \left\langle \omega, \bsV_t^{-1/2} \bsS_t \right\rangle \\
            &\leq \frac{\epsilon}{\sqrt{\Bar{\kappa}}} \sqrt{\Bar{\kappa}} \left\| \bsV_t^{-1/2} \bsS_t \right\|_2 + \sup_{\nu \in K_0} \left\langle \frac{\bsV_t^{1/2} \nu}{\| \bsV_t^{1/2} \nu \|_2}, \bsV_t^{-1/2} \bsS_t \right\rangle \\
            &= \epsilon \left\| \bsV_t^{-1/2} \bsS_t \right\|_2 + \sup_{\nu \in K_0} \frac{ \langle \nu, \bsS_t \rangle}{\sqrt{\langle \nu, \bsV_t \nu \rangle}},
        \end{align*}
        where the second inequality uses Lemma 7.1 of \cite{whitehouse2023timeuniform} and the definition of $K_0(t)$. Combining \eqref{eq: lil en union bd}, we have with probability at least $1-\alpha$, for any $t\geq 1$,
        \begin{align*}
            \left\| \bsV_t^{-1/2} \bsS_t \right\|_2 &\leq \frac{1}{1-\epsilon} \sup_{\nu \in K_0} \frac{ \langle \nu, \bsS_t \rangle}{\sqrt{\langle \nu, \bsV_t \nu \rangle}} \\
            &\leq \frac{1}{1-\epsilon} \sqrt{2\eta \left\{ s\log\log (\eta \Bar{\kappa} t) + \log \left( \frac{N_0 \zeta(s)}{\alpha \log^s \eta} \right) \right\}},
        \end{align*}
        which concludes the proof.
    \end{proof}

    In our case, $\bsS_t = t \bsM_t$ and $\bsV = \bsH^{-1}\bsS \bsH^{-1}$. By Lemma 4.2 of \cite{whitehouse2023timeuniform}, $N_0$ can be bounded by $N_0 \leq C_d (3\sqrt{\Bar{\kappa}} / \epsilon)^{d-1}$, where the constant $C_d = d2^d \Gamma\big(\frac{d+1}{2}\big) / \pi^{\frac{d-1}{2}}$. By letting $\eta =2$ and $s=1.4$ in Lemma \ref{lem: lil en multi}, which are also the same hyperparameter choices as in Appendix \ref{app: lil ub}, we derive the final boundary as
    \begin{align*}
        \left\| (\bsH^{-1} \bsS \bsH^{-1})^{-1/2} \bsM_t \right\|_2 \leq \frac{2}{1-\epsilon} \sqrt{ \frac{1.4 \log\log(2 \Bar{\kappa} t) + \log\left( \frac{5.2 C_d}{\alpha} \right) + (d-1) \log\left( \frac{3\sqrt{\Bar{\kappa}}}{\epsilon} \right)}{t}}.
    \end{align*}

\end{proof}

\subsection{Proof of Theorem \ref{thm: asymp t1 control}}

    Write $\Bar{\bsx}_t - \bsx^* \in C_{\widehat{\bsH}_t^{-1}\widehat{\bsS}_t \widehat{\bsH}_t^{-1}}^{\clubsuit}(t)$ as $\left\| (\widehat{\bsH}_t^{-1}\widehat{\bsS}_t \widehat{\bsH}_t^{-1})^{-1/2} (\Bar{\bsx}_t - \bsx^*) \right\|_{\mathfrak{p}} \leq \operatorname{Bnd}^\clubsuit_{\widehat{\bsH}_t^{-1}\widehat{\bsS}_t \widehat{\bsH}_t^{-1}}(t)$ where $\mathfrak{p} \in \{2,\infty\}$ according to Table \ref{tab: cs}. By Theorem \ref{thm: strong approx} and Assumption \ref{ass: step size rate}, there exists $\bsM_t = \frac{1}{t}\sum_{j=1}^t \bsG_j$ where $\bsG_1,\bsG_2,\dots$ are i.i.d. Gaussians with mean zero and covariance $\bsH^{-1}\bsS \bsH^{-1}$ such that $\Bar{\bsx}_t - \bsx^* = \bsM_t + \tilde{o}(t^{-1/2})$. In addition, by Proposition \ref{prop: as coverge HSH}, $\widehat{\bsH}_t^{-1}\widehat{\bsS}_t \widehat{\bsH}_t^{-1} = \bsH^{-1} \bsS \bsH^{-1} + o(1)$. These together imply
    \begin{align}
    \label{eq: type1 control 1}
        \frac{\left\| (\widehat{\bsH}_t^{-1}\widehat{\bsS}_t \widehat{\bsH}_t^{-1})^{-1/2} (\Bar{\bsx}_t - \bsx^*) \right\|_{\mathfrak{p}} }{ \operatorname{Bnd}^\clubsuit_{\widehat{\bsH}_t^{-1}\widehat{\bsS}_t \widehat{\bsH}_t^{-1}}(t)} = \frac{\left\| ({\bsH}^{-1}{\bsS} {\bsH}^{-1})^{-1/2} \bsM_t \right\|_{\mathfrak{p}} }{ \operatorname{Bnd}^\clubsuit_{{\bsH}^{-1}{\bsS} {\bsH}^{-1}}(t)}\{ 1 + o(1) \}.
    \end{align}
    We explicitly bound the above $o(1)$ term by a positive decreasing random variable $\gE_t \searrow 0$. Denote the events
    \begin{align*}
        \gA_m &:= \left\{ \exists t\geq m\colon \frac{\left\| ({\bsH}^{-1}{\bsS} {\bsH}^{-1})^{-1/2} \bsM_t \right\|_{\mathfrak{p}} }{ \operatorname{Bnd}^\clubsuit_{{\bsH}^{-1}{\bsS} {\bsH}^{-1}}(t)} \geq 1 \right\}, \\
        \widehat{\gA}_m &:= \left\{ \exists t\geq m\colon \frac{\left\| (\widehat{\bsH}_t^{-1}\widehat{\bsS}_t \widehat{\bsH}_t^{-1})^{-1/2} (\Bar{\bsx}_t - \bsx^*) \right\|_{\mathfrak{p}} }{ \operatorname{Bnd}^\clubsuit_{\widehat{\bsH}_t^{-1}\widehat{\bsS}_t \widehat{\bsH}_t^{-1}}(t)} \geq 1 - \gE_t \right\}.
    \end{align*}
    According to \eqref{eq: type1 control 1}, $\gA_m \subseteq \widehat{\gA}_m$, and
    \begin{align}
    \label{eq: small event}
        \lim_{m\rightarrow \infty} \PB(\widehat{\gA}_m \backslash \gA_m) &\leq \lim_{m \rightarrow \infty} \PB\left( \exists t\geq m\colon \frac{\left\| ({\bsH}^{-1}{\bsS} {\bsH}^{-1})^{-1/2} \bsM_t \right\|_{\mathfrak{p}} }{ \operatorname{Bnd}^\clubsuit_{{\bsH}^{-1}{\bsS} {\bsH}^{-1}}(t)} \geq \frac{1 - \gE_t}{1 + \gE_t} \right).
    \end{align}
    Our goal is to prove that $\PB(\widehat{\gA}_m \backslash \gA_m) \rightarrow 0$, so that the probability of events with plug-in estimators can be bounded by the probability of events with known covariances.
    
    To proceed, we need a law of the iterated logarithm for $({\bsH}^{-1}{\bsS} {\bsH}^{-1})^{-1/2} \bsM_t$ from the following lemma.
    \begin{lemma}
    \label{lem: multi lil gauss}
        Let $\bsG_1,\bsG_2,\dots$ be i.i.d. Gaussian random vectors with mean zero and covariance $\bsI$, and let $\bsS_t:=\sum_{j=1}^t \bsG_j$. Then for $\mathfrak{p} \in \{2,\infty\}$,
        \begin{align*}
            \limsup_{t\rightarrow \infty} \frac{\|\bsS_t\|_\mathfrak{p}}{\sqrt{2t\log\log t}} = 1.
        \end{align*}
    \end{lemma}
    \begin{proof}[Proof of Lemma \ref{lem: multi lil gauss}]
        For $\mathfrak{p}=\infty$, we use the fact that each coordinate of $\bsS_t$ is a sum of standard Gaussian variables to obtain
        \begin{align*}
            \limsup_{t\rightarrow \infty} \frac{\|\bsS_t\|_\infty}{\sqrt{2t\log\log t}} &= \limsup_{t\rightarrow \infty} \max_{1\leq i\leq d} \frac{|[\bsS_t]_i|}{\sqrt{2t\log\log t}} \\
            &= \max_{1\leq i\leq d}  \limsup_{t\rightarrow \infty} \frac{|[\bsS_t]_i|}{\sqrt{2t\log\log t}} \\
            &= \max_{1\leq i\leq d} 1 = 1.
        \end{align*}
        For $\mathfrak{p}=2$, we use Theorem 3 of \cite{berning1979multivariate}, a law of the iterated logarithm in Hilbert spaces, to obtain that the limit set of $\frac{\bsS_t}{\sqrt{2t\log\log t}}$ is $\SB^{d-1}$ almost surely, which implies
        \begin{align*}
            \limsup_{t\rightarrow \infty} \frac{\|\bsS_t\|_2}{\sqrt{2t\log\log t}} = 1 \quad \text{a.s.}
        \end{align*}
    \end{proof}
    Back to the main proof, letting $\bsS_t = t (\bsH^{-1} \bsS \bsH^{-1})^{-1/2} \bsM_t$ in Lemma \ref{lem: multi lil gauss}, there exists a positive decreasing random variable $\gE'_t \searrow 0$ such that
    \begin{align*}
        \frac{\left\| ({\bsH}^{-1}{\bsS} {\bsH}^{-1})^{-1/2} \bsM_t \right\|_{\mathfrak{p}} }{\sqrt{\frac{2\log\log t}{t}}} \leq 1 + \gE'_t.
    \end{align*}
    On the other hand, it is clear from the definitions of three confidence sequence in Table \ref{tab: cs} that there exist universal constants $c<1$ and $T_0 > 0$ such that for any $\clubsuit \in \{\text{LIL-UB}, \text{GM}, \text{LIL-EN}\}$ and any $t\geq T_0$, we have $\sqrt{\frac{2\log\log t}{t}} \leq c \cdot \operatorname{Bnd}^{\clubsuit}_{\bsH^{-1}\bsS \bsH^{-1}}(t) $. Therefore, the limit probability in \eqref{eq: small event} can further bounded by
    \begin{align}
        \lim_{m\rightarrow \infty} \PB(\widehat{\gA}_m \backslash \gA_m) &\leq \lim_{m \rightarrow \infty} \PB\left( \exists t\geq m\colon \frac{\left\| ({\bsH}^{-1}{\bsS} {\bsH}^{-1})^{-1/2} \bsM_t \right\|_{\mathfrak{p}} }{ \operatorname{Bnd}^\clubsuit_{{\bsH}^{-1}{\bsS} {\bsH}^{-1}}(t)} \geq \frac{1 - \gE_t}{1 + \gE_t} \right) \notag \\
        &\leq \lim_{m \rightarrow \infty} \PB\left( \exists t\geq m\colon \frac{\left\| ({\bsH}^{-1}{\bsS} {\bsH}^{-1})^{-1/2} \bsM_t \right\|_{\mathfrak{p}} }{ \sqrt{\frac{2\log\log t}{t}}} \geq \frac{1}{c} \frac{1 - \gE_t}{1 + \gE_t} \right) \notag \\
        &\leq \lim_{m \rightarrow \infty} \PB\left( \exists t\geq m\colon 1 + \gE'_t \geq \frac{1}{c} \frac{1 - \gE_t}{1 + \gE_t} \right) \notag \\
        &= \lim_{m \rightarrow \infty} \PB\left( \exists t\geq m\colon c \geq  \frac{1}{1 + \gE'_t} \frac{1 - \gE_t}{1 + \gE_t} \right) \notag \\
        &\leq \lim_{m \rightarrow \infty} \PB\left( c \geq  \frac{1}{1 + \gE'_m} \frac{1 - \gE_m}{1 + \gE_m} \right)  = 0, \label{eq: Am tilde Am}
    \end{align}
    where the last inequality holds because $\frac{1}{1 + \gE'_t} \frac{1 - \gE_t}{1 + \gE_t} \nearrow 1$ is increasing by the definitions of $\gE_t$ and $\gE'_t$. Finally, we conclude the proof by
    \begin{align*}
        &\quad\ \lim_{m\rightarrow \infty} \PB \left( \exists t\geq m\colon \Bar{\bsx}_t - \bsx^* \not\in C^\clubsuit_{\widehat{\bsH}^{-1}_t \widehat{\bsS}_t \widehat{\bsH}^{-1}_t}(t) \right) \\
        &= \lim_{m\rightarrow \infty} \PB \left( \exists t\geq m\colon \frac{\left\| (\widehat{\bsH}^{-1}_t \widehat{\bsS}_t \widehat{\bsH}^{-1}_t)^{-1/2} (\Bar{\bsx}_t - \bsx^*) \right\|_\mathfrak{p}}{\operatorname{Bnd}^\clubsuit_{\widehat{\bsH}^{-1}_t \widehat{\bsS}_t \widehat{\bsH}^{-1}_t}(t)} \geq 1  \right) \\
        &\overset{(a)}{\leq} \lim_{m\rightarrow \infty} \PB(\widehat{\gA}_m) \overset{(b)}{=} \lim_{m\rightarrow \infty} \PB(\gA_m) \overset{(c)}{\leq} \PB\left( \exists t\geq 1\colon \bsM_t \not\in C^\clubsuit_{\bsH^{-1} \bsS \bsH^{-1}}(t) \right) \overset{(d)}{\leq} \alpha,
    \end{align*}
    where $(a)$ holds because $\gE_t > 0$, $(b)$ follows from \eqref{eq: Am tilde Am}, $(c)$ holds because $\gA_m$'s are decreasing events, and $(d)$ follows from Proposition \ref{prop: general boundary gaussian}.
\section{Auxiliary Lemmas and Proofs}

\subsection{Proof of Lemma \ref{lem: as clt approx}}
\label{sec: lem: as clt approx}
Denote $\bsDelta_t := \bsx_t - \bsx^*$, $\bsr_t:=\bsH(\bsx_t - \bsx^*) - \bsg(\bsx_t)$ and $\bsvareps_{t+1} := \bsg(\bsx_t) - \bsG(\bsx_t,\xi_{t+1})$, as well as the averaged iterates $\Bar{\bsx}_t := \frac{1}{t}\sum_{s=1}^t \bsx_s$ and $\Bar{\bsDelta}_t := \frac{1}{t} \sum_{s=1}^t \bsDelta_s$. We decompose the recursion \eqref{eq: sa} as
\begin{align*}
    \bsDelta_{t+1} &= \bsDelta_t - \eta_t \bsG(\bsx_t,\xi_{t+1}) \\
    &= \bsDelta_t - \eta_t \bsH(\bsx_t - \bsx^*) - \eta_t \{\bsg(\bsx_t) - \bsH(\bsx_t - \bsx^*)\} - \eta_t \{\bsG(\bsx_t,\xi_{t+1}) - \bsg(\bsx_t)\} \\
    &= (\bsI - \eta_t \bsH) \bsDelta_t + \eta_t \bsr_t + \eta_t \bsvareps_{t+1}.
\end{align*}
Further denote $\bsB_t := (\bsI - \eta_t \bsH)$ and $\bsA_j^t := \sum_{s=j}^t \left( \prod_{i=j+1}^s \bsB_i \right) \eta_j$ for any $j\leq t$. Recurring the above equation and summing up $\Delta_t$'s yield
\begin{align*}
    \sum_{s=0}^t \bsDelta_{s+1} &= \sum_{s=0}^t \left\{ \left( \prod_{j=0}^{s} (\bsI - \eta_j \bsH) \right) \bsDelta_0 + \sum_{j=0}^{s} \left( \prod_{i=j+1}^{s} (\bsI - \eta_i \bsH) \right) \eta_j \bsr_j + \sum_{j=0}^{s} \left( \prod_{i=j+1}^{s} (\bsI - \eta_i \bsH) \right) \eta_j \bsvareps_{j+1} \right\} \\
    &= \sum_{s=0}^t \left( \prod_{j=0}^s \bsB_j \right) \bsDelta_0 + \sum_{s=0}^t \sum_{j=0}^s \left(\prod_{i=j+1}^s \bsB_i \right) \eta_j \bsr_j + \sum_{s=0}^t \sum_{j=0}^s \left(\prod_{i=j+1}^s \bsB_i \right) \eta_j \bsvareps_{j+1} \\
    &= \sum_{s=0}^t \left( \prod_{j=0}^s \bsB_j \right) \bsDelta_0 + \sum_{j=0}^t \sum_{s=j}^t \left(\prod_{i=j+1}^s \bsB_i \right) \eta_j \bsr_j + \sum_{j=0}^t \sum_{s=j}^t \left(\prod_{i=j+1}^s \bsB_i \right) \eta_j \bsvareps_{j+1} \\
    &= \bsA_0^t \bsB_0 \bsDelta_0 + \sum_{j=0}^t \bsA_j^t \bsr_j + \sum_{j=0}^t \bsA_j^t \bsvareps_{j+1} \\
    &= \bsA_0^t \bsB_0 \bsDelta_0 + \sum_{j=0}^t \bsA_j^t \bsr_j + \sum_{j=0}^t (\bsA_j^t - \bsH^{-1}) \bsvareps_{j+1} + \sum_{j=0}^t \bsH^{-1} \bsvareps_{j+1}.
\end{align*}
This further implies
\begin{align}
\label{eq: decompose}
    \Bar{\bsx}_{t} - \bsx^* = \Bar{\bsDelta}_{t} = \frac{1}{t}\sum_{s=1}^t \bsDelta_s = \underbrace{\frac{1}{t} \bsA_0^{t-1} \bsB_0 \bsDelta_0}_{=:\gT_1} + \underbrace{\frac{1}{t} \sum_{j=0}^{t-1} \bsA_j^{t-1} \bsr_j}_{=: \gT_2} + \underbrace{\frac{1}{t} \sum_{j=0}^{t-1} (\bsA_j^{t-1} - \bsH^{-1}) \bsvareps_{j+1}}_{=: \gT_3} + \underbrace{\frac{1}{t} \sum_{j=0}^{t-1} \bsH^{-1} \bsvareps_{j+1}}_{=: \gT_0}.
\end{align}
For notational simplicity, we denote the four terms in \eqref{eq: decompose} as $\gT_0$, $\gT_1$, $\gT_2$, $\gT_3$. Note that $\gT_0$ is the noise term that contributes to the ``well-behaved'' randomness, so our goal is to establish the almost-sure convergence rates for the negligible terms $\gT_1$, $\gT_2$, $\gT_3$. According to the discussions below, we will finally arrive at the following approximation result: for arbitrarily small $\epsilon>0$,
\begin{align*}
    \Bar{\bsx}_t - \bsx^* &= \frac{1}{t} \sum_{j=1}^{t} \bsH^{-1} \bsvareps_{j} + O\left(\underbrace{t^{-1}}_{\gT_1} + \underbrace{t^{\frac{1}{2p} - \frac{2-a}{2}}(\log t)^{\frac{1}{2p}+\epsilon}}_{\gT_3}\right) + o\left(\underbrace{t^{-\frac{a(1+\lambda)}{2}}(\log t)^{1+\epsilon}}_{\gT_2}\right) \quad \text{a.s.}\\
    &= \frac{1}{t} \sum_{j=1}^{t} \bsH^{-1} \bsvareps_{j} + o\left( t^{-1}(\log t)^\epsilon + t^{-\frac{a(1+\lambda)}{2}}(\log t)^{1+\epsilon} + t^{\frac{1}{2p} - \frac{2-a}{2}}(\log t)^{\frac{1}{2p}+\epsilon}\right) \quad \text{a.s.},
\end{align*}
where the last equality holds because 
$\epsilon>0$ is arbitrary.

\paragraph{Almost-sure convergence of $\gT_1$.} According to Lemma \ref{lem: bounded matrix}, $\bsA_0^{t-1}$ is uniformly bounded by $\|\bsA_0^{t-1}\| \leq C_0$, which implies $\|\gT_1\| \leq \frac{1}{t} C_0 \|\bsB_0 \bsDelta_0\|$. Since $\|\bsB_0 \bsDelta_0\|$ is a random variable that does not vary with time, it follows that $\|\gT_1\| = O(t^{-1})$ almost surely.

\paragraph{Almost-sure convergence of $\gT_2$.} According to Lemma \ref{lem: bounded matrix}, each $\bsA_j^{t-1}$ is uniformly bounded by $\|\bsA_j^{t-1}\| \leq C_0$, which implies $\|\gT_2\| \leq \frac{1}{t} C_0 \sum_{j=0}^{t-1} \|\bsr_j\|$. Then by the following lemma, we obtain an almost-sure convergence rate $\|\gT_2\| = o\left(t^{-\frac{a(1+\lambda)}{2}}(\log t)^{1+\epsilon}\right)$ for any small constant $\epsilon>0$.

\begin{lemma}
\label{lem: as T2}
    Under Assumptions \ref{ass: lyapunov}-\ref{ass: noise}, it holds that $\sum_{j=0}^{t-1} \|\bsr_j\| = o\left(t^{1-\frac{a(1+\lambda)}{2}}(\log t)^{1+\epsilon}\right)$ almost surely for any small constant $\epsilon>0$.
\end{lemma}

\paragraph{Almost-sure convergence of $\gT_3$.} We have
\begin{align*}
    \E\big(\| \gT_3 \|^{2p}\big) &= \frac{1}{t^{2p}} \E\left[ \left\| \sum_{j=0}^{t-1} (\bsA_j^{t-1} - \bsH^{-1}) \bsvareps_{j+1} \right\|^{2p} \right] \\
    &\leq \frac{C_{2p}}{t^{2p}} \E \left| \sum_{j=0}^{t-1} \|\bsA_j^{t-1} - \bsH^{-1}\|^2 \|\bsvareps_{j+1}\|^2 \right|^p \\
    &= \frac{C_{2p}}{t^{2p}} \sum_{j_1,\dots,j_p = 0}^{t-1} \left( \prod_{i=1}^p \|\bsA_{j_i}^{t-1} - \bsH^{-1}\|^2 \right) \E\left( \prod_{i=1}^p \|\bsvareps_{j_i+1}\|^2 \right),
\end{align*}
where the inequality above is due to Burkholder's inequality (e.g., Theorem 2.10 of \cite{hall2014martingale}) which we restate in Lemma \ref{lem: burkholder} for reference, and $C_{2p}$ is a universal constant that depends only on $p$. Note that $\E\big( \prod_{i=1}^p \|\bsvareps_{j_i+1}\|^2 \big)$ is bounded since $\E\big(\|\bsvareps_{j+1}\|^{2p}\big)$ is bounded by Assumption \ref{ass: noise} \ref{ass: noise 3}. In addition, by Lemma \ref{lem: bounded matrix},
\begin{align*}
    \sum_{j=0}^{t-1} \| \bsA_j^{t-1} - \bsH^{-1} \|^2 \leq (C_0 + \|\bsH^{-1}\|) \sum_{j=0}^{t-1} \| \bsA_j^{t-1} - \bsH^{-1} \|.
\end{align*}
According to Lemma \ref{lem: convergence matrix}, $\sum_{j=0}^{t-1} \| \bsA_j^{t-1} - \bsH^{-1} \| = O(t^a)$, so we obtain
\begin{align*}
    \E\big(\| \gT_3 \|^{2p}\big) &\lesssim \frac{1}{t^{2p}} \left( \sum_{j=0}^{t-1} \|\bsA_j^{t-1} - \bsH^{-1}\|^2 \right)^p \lesssim \frac{1}{t^{2p}} \left( \sum_{j=0}^{t-1} \|\bsA_j^{t-1} - \bsH^{-1}\| \right)^p = O(t^{(a-2)p}).
\end{align*}
Now fix an arbitrarily small quantity $\epsilon>0$ and a constant $C>0$, and denote more precisely $\gT_{3,t}$ as the quantity $\gT_3$ at time $t$. By Markov's inequality,
\begin{align*}
    \PB\left(t^{\frac{2-a}{2} - \frac{1}{2p}}(\log t)^{-\frac{1}{2p} - \epsilon}\|\gT_{3,t}\| \geq C\right) &\leq C^{-2p} t^{(2-a)p - 1} (\log t)^{-1-2p\epsilon} \E\big(\| \gT_3 \|^{2p}\big) \lesssim C^{-2p} t^{-1}(\log t)^{-1 - 2p\epsilon},
\end{align*}
and since $\sum_{t=2}^\infty t^{-1}(\log t)^{-1 - 2p\epsilon} < \infty$, by the Borel-Cantelli lemma, we obtain
\begin{align*}
    \PB\left(t^{\frac{2-a}{2} - \frac{1}{2p}}(\log t)^{-\frac{1}{2p} - \epsilon}\|\gT_{3,t}\| \geq C \quad \text{i.o. }\right) = 0 \quad \Longrightarrow \quad \|\gT_{3,t}\| = O\left(t^{\frac{1}{2p} - \frac{2-a}{2}}(\log t)^{\frac{1}{2p}+\epsilon}\right) \quad \text{a.s.}
\end{align*}
A similar but more refined analysis can yield $\|\gT_{3}\| = \|\gT_{3,t}\| = O\left(t^{\frac{1}{2p} - \frac{2-a}{2}}(\log t)^{\frac{1}{2p}}(\log\log t)^{\frac{1}{2p} + \epsilon}\right)$ almost surely.

\subsection{Proof of Lemma \ref{lem: strong approx}}
\label{sec: lem: strong approx}
Recall that $\bsvareps_t = \bsg(\bsx_{t-1}) - \bsG(\bsx_{t-1},\xi_t)$, $\bsS(\bsx) = \var(\bsG(\bsx,\xi))$ and $\bsS=\bsS(\bsx^*)$, so we have $\var(\bsH^{-1}\bsvareps_t\mid \gF_{t-1}) = \bsH^{-1}\bsS(\bsx_{t-1})\bsH^{-1} \xrightarrow[]{a.s.} \bsH^{-1} \bsS \bsH^{-1}$. Our goal is to establish an almost-sure invariance principle (ASIP) for $\frac{1}{t}\sum_{j=1}^t \bsH^{-1} \bsvareps_j$, i.e., to prove existence of a Brownian motion on the same (potentially enriched) probability space as the data points $\{\xi_t\}_{t\geq 1}$.

In the one-dimensional case where $d=1$, such an ASIP result was established by \cite{strassen1967almost} via the famous Skorokhod embedding theorem. We present one of the main theorems in \cite{strassen1967almost} in Lemma \ref{lem: asip}. However, extending it to the multi-dimensional case can be difficult as \cite{monrad1991problem} showed that one can not embed a general $\R^d$-valued martingale in an $\R^d$-valued Gaussian process, unless additional assumptions related to the limiting behavior of the covariance matrix are required \citep{morrow1982almost,philipp1986note}. Although those assumptions are likely to hold since the (conditional) covariance matrix $\bsH^{-1}\bsS(\bsx_{t-1})\bsH^{-1}$ in our problem converges to $\bsH^{-1}\bsS\bsH^{-1}$ almost surely, directly applying ASIPs in \cite{morrow1982almost,philipp1986note} does not leverage the specific variance structure of $\frac{1}{t}\sum_{j=1}^t \bsH^{-1} \bsvareps_j$ in our problem and may only lead to a coarse almost-sure convergence rate. Instead, here we first establish strong approximation to an $\R^d$-valued martingale of i.i.d. sums using the $L^2$ convergence property of the iterates $\{\bsx_t\}_{t\geq 0}$, and then apply Lemma \ref{lem: asip hilbert}, an ASIP theorem in Hilbert spaces, to establish strong approximation to an $\R^d$-valued Brownian motion.

Denote $\Tilde{\bsvareps}_t := \bsg(\bsx^*) - \bsG(\bsx^*,\xi_t) = - \bsG(\bsx^*,\xi_t)$, $\bsGamma_t := \frac{1}{t}\sum_{j=1}^t \bsH^{-1} \bsvareps_j$ and $\widetilde{\bsGamma}_t :=  \frac{1}{t}\sum_{j=1}^t \bsH^{-1} \Tilde{\bsvareps}_j$. Note that $\{t\bsGamma_t\}_{t\geq 0}$ and $\{t\widetilde{\bsGamma}_t\}_{t\geq 0}$ are both martingales. We first prove $\bsGamma_t - \widetilde{\bsGamma}_t$ shrinks to zero almost surely, and then apply the one-dimensional ASIP, Lemma \ref{lem: asip}, to $\widetilde{\bsGamma}_t$. The proof is finally complete by combing the two steps.

Note that
\begin{align}
\label{eq: gamma diff}
    \bsGamma_t - \widetilde{\bsGamma}_t = \frac{1}{t} \sum_{j=1}^t \bsH^{-1} (\bsvareps_j - \Tilde{\bsvareps}_j) = \frac{1}{t} \sum_{j=1}^t \bsH^{-1} (\bsg(\bsx_{j-1}) - \bsG(\bsx_{j-1}, \xi_j) + \bsG(\bsx^*,\xi_j)).
\end{align}
The covariance matrix of each term on the right-hand side of \eqref{eq: gamma diff} is dominated by
\begin{align}
    &\quad\ \|\bsH^{-1}\|^2 \cdot \E\|\bsg(\bsx_{j-1}) - \bsG(\bsx_{j-1}, \xi_j) + \bsG(\bsx^*,\xi_j)\|^2 \notag \\
    &\leq \|\bsH^{-1}\|^2 \left( \E\|\bsg(\bsx_{j-1}) - \bsG(\bsx_{j-1}, \xi_j) + \bsG(\bsx^*,\xi_j)\|^2 \right) \notag \\
    &\leq 2 \|\bsH^{-1}\|^2 \left( \E\|\bsg(\bsx_{j-1})\|^2 + \E\|\bsG(\bsx_{j-1}, \xi_j) - \bsG(\bsx^*,\xi_j)\|^2 \right) \notag \\
    &\overset{(a)}{\leq} 2 \|\bsH^{-1}\|^2 \left( C_G \E\|\bsx_{j-1} - \bsx^*\|^2 + L_G^2 \E\|\bsx_{j-1} - \bsx^*\|^2 \right) \notag \\
    &\overset{(b)}{\leq} 2\widetilde{C}_0  (C_G + L_G^2) \|\bsH^{-1}\|^2 \eta_{j-1}, \label{eq: bias cov bound rate}
\end{align}
where (a) uses Lemma \ref{lem: bound g L2} and Assumption \ref{ass: lipschitz}, and (b) uses Lemma \ref{lem: as l2 conv}.
\begin{lemma}
    \label{lem: bound g L2}
    Under Assumptions \ref{ass: local linear} and \ref{ass: noise}, there exists a constant $C_G>0$ such that for all $t\geq 0$, $\E\|\bsg(\bsx_t)\|^2 \leq C_G \E \|\bsx_t - \bsx^*\|^2$.
\end{lemma}

For notational simplicity, denote $C_\Gamma := 2\widetilde{C}_0  (C_G + L_G^2) \|\bsH^{-1}\|^2$. \eqref{eq: bias cov bound rate} implies for any vector $\ell\in \R^d$ with $\|\ell\|^2 = 1$, $\var\big\{\ell^\top \bsH^{-1} (\bsvareps_t - \Tilde{\bsvareps}_t)\big\} \leq C_\Gamma \eta_t$. Therefore, for arbitrarily small $\epsilon>0$ we have 
\begin{align*}
    \sum_{t=2}^\infty \var \left\{ \frac{\ell^\top \bsH^{-1} (\bsvareps_t - \Tilde{\bsvareps}_t)}{t^{\frac{1-a}{2}}(\log t)^{\frac{1}{2} + \epsilon}} \right\} \leq \sum_{t=2}^\infty \frac{C_\Gamma \eta_t}{t^{1-a} (\log t)^{1+2\epsilon}} = \sum_{t=2}^\infty \frac{C_\Gamma \eta_0}{t (\log t)^{1+2\epsilon}} < \infty.
\end{align*}
By Theorem 2.5.6 of \cite{durrett2019probability} \footnote{The original theorem is stated for a sum of independent variables; however, it can be easily extended for a martingale sequence, since the main tool used for proof is Kolmogorov's maximal inequality which is valid for martingales.}, with probability one it holds that
\begin{align*}
    \sum_{t=2}^\infty \frac{\ell^\top \bsH^{-1} (\bsvareps_t - \Tilde{\bsvareps}_t)}{t^{\frac{1-a}{2}}(\log t)^{\frac{1}{2} + \epsilon}} < \infty.
\end{align*}
By Kronecker's lemma, we obtain with probability one,
\begin{align}
\label{eq: strong approx proof 1}
    \frac{1}{t^{\frac{1-a}{2}}(\log t)^{\frac{1}{2} + \epsilon}} \sum_{j=1}^t \ell^\top \bsH^{-1} (\bsvareps_j - \Tilde{\bsvareps}_j) \rightarrow 0 \quad \Longrightarrow \quad \bsGamma_t - \widetilde{\bsGamma}_t = o \left(t^{-\frac{1+a}{2}}(\log t)^{\frac{1}{2}+\epsilon} \right).
\end{align}

Next, we approximate $\{\widetilde{\bsGamma}_t\}_{t\geq 0}$ with a Brownian motion using Lemma \ref{lem: asip hilbert}. Taking $\HB = \R^d$ and $\bsX_n = \bsH^{-1} \Tilde{\bsvareps}_n$ in Lemma \ref{lem: asip hilbert}, we immediately have that $\bsA_n = n (\bsH^{-1} \bsS \bsH^{-1})$ and $V_n = n\cdot \tr(\bsH^{-1} \bsS \bsH^{-1})$. Hence, we can let $\bsA = \bsH^{-1} \bsS \bsH^{-1} / \tr(\bsH^{-1} \bsS \bsH^{-1})$ so that \eqref{eq: asip hirbert 2} is automatically satisfied with $q=1$. Choose $f(t) = t^{\frac{1}{p}}(\log t)^{\frac{1}{p}+\epsilon}$ for an arbitrarily small $\epsilon>0$ so the requirements for $f$ are met. In addition, by Assumption \ref{ass: noise} \ref{ass: noise 3}, 
\begin{align*}
    &\quad \ \sum_{n=1}^\infty \E\left[ \frac{\|\bsH^{-1} \Tilde{\bsvareps}_n\|^2 \indicator\{\|\bsH^{-1} \Tilde{\bsvareps}_n\|^2 > f(n\cdot \tr(\bsH^{-1} \bsS \bsH^{-1}))\}}{f(n\cdot \tr(\bsH^{-1} \bsS \bsH^{-1}))} \right] \\
    &\leq \sum_{n=1}^\infty \E\left[ \frac{\|\bsH^{-1} \Tilde{\bsvareps}_n\|^{2p} \indicator\{\|\bsH^{-1} \Tilde{\bsvareps}_n\|^2 > f(n\cdot \tr(\bsH^{-1} \bsS \bsH^{-1}))\}}{f(n\cdot \tr(\bsH^{-1} \bsS \bsH^{-1}))^p} \right] \\
    &\leq \sum_{n=1}^\infty \frac{ \E\|\bsH^{-1} \Tilde{\bsvareps}_n\|^{2p}}{f(n\cdot \tr(\bsH^{-1} \bsS \bsH^{-1}))^p} \lesssim \sum_{n=1}^\infty \frac{1}{n (\log n)^{1 + p\epsilon}} < \infty,
\end{align*}
which verifies \eqref{eq: asip hirbert 1}. Therefore, there exists a $d$-dimensional Brownian motion $\bsB = (\bsB_t)_{t\geq 0}$ in the same (potentially enriched) probability space with $\bsB_t - \bsB_s \sim \gN(\0, (t-s)\bsA)$ such that for arbitrarily small $\epsilon>0$,
\begin{align}
\label{eq: after apply asip}
    \sum_{i=1}^n \bsH^{-1} \Tilde{\bsvareps}_i - \bsB(n \cdot \tr(\bsH^{-1} \bsS \bsH^{-1})) = O\big( n^{\frac{1}{2} - \frac{p-1}{50dp}} (\log n)^{\frac{1}{50dp} + \epsilon} \big) \quad \text{a.s.}
\end{align}
Letting $\bsG_n = \bsB(n \cdot \tr(\bsH^{-1} \bsS \bsH^{-1})) - \bsB((n-1) \cdot \tr(\bsH^{-1} \bsS \bsH^{-1}))$, we have that $\bsG_n$'s are i.i.d. centered Gaussian with covariance $\bsH^{-1} \bsS \bsH^{-1}$. Since $\epsilon>0$ is chosen arbitrarily, \eqref{eq: after apply asip} implies
\begin{align}
\label{eq: strong approx proof 2}
    \widetilde{\bsGamma}_t - \frac{1}{t}\sum_{j=1}^t \bsG_j = o\big( t^{-\frac{1}{2} - \frac{p-1}{50dp}} (\log t)^{\frac{1}{50dp} + \epsilon} \big) \quad \text{a.s.}
\end{align}
Combining \eqref{eq: strong approx proof 1} and \eqref{eq: strong approx proof 2} concludes the proof. Note that when $d=1$, with the same choice of $f(t) = t^{\frac{1}{p}}(\log t)^{\frac{1}{p}+\epsilon}$, we can apply Lemma \ref{lem: asip} to obtain a slightly better convergence rate $o\big( t^{-\frac{1}{2} - \frac{p-1}{4p}} (\log t)^{1 + \frac{1}{4p} + \epsilon} \big)$.

\subsection{Proof of Lemma \ref{lem: as T2}}
\label{sec: lem: as T2}
This proof mainly follows Part 4 in the proof of \cite{polyak1992acceleration}'s Theorem 2.
To proceed, we first need to establish the almost-sure and $L^2$ convergence of $\bsx_t$ as follows.

\begin{lemma}[Almost-sure and $L^2$ convergence]
\label{lem: as l2 conv}
    Under Assumptions \ref{ass: lyapunov}, \ref{ass: step size} and \ref{ass: noise}, $\bsx_t \rightarrow \bsx^*$ almost surely. In addition, there exists a universal constant $\widetilde{C}_0>0$ such that for any $t\geq 0$, $\E\|\bsx_t -\bsx^*\|^2 \leq \widetilde{C}_0 \eta_t$.
\end{lemma}

Recall that $\bsr_t = \bsH(\bsx_t - \bsx^*) - \bsg(\bsx_t)$. According to Assumptions \ref{ass: local linear} and \ref{ass: noise} \ref{ass: noise 1}, we have
\begin{align*}
    \|\bsr_t \| &= \|\bsH(\bsx_t - \bsx^*) - \bsg(\bsx_t)\| \\
    &\leq \left\{ \begin{aligned}
        & L_H \|\bsx_t - \bsx^*\|^{1+\lambda} &&\text{if } \|\bsx_t - \bsx^*\| \leq \delta_H \\
        & \|\bsH\| \|\bsx_t - \bsx^*\| + \sqrt{L_\varepsilon (1 + \|\bsx_{t} - \bsx^*\|^2)} &&\text{if } \|\bsx_t - \bsx^*\| > \delta_H
    \end{aligned} \right. \\
    &\leq \max\left\{ L_H, \frac{\|\bsH\| + \sqrt{L_\varepsilon \left( 1 + \frac{1}{\delta_H^2} \right)}}{\delta_H^\lambda} \right\} \|\bsx_t - \bsx^*\|^{1+\lambda} \\
    &=: L_R \|\bsx_t - \bsx^*\|^{1+\lambda}.
\end{align*}
Let $\epsilon>0$ be an arbitrarily small constant. Using the $L^2$ convergence result in Lemma \ref{lem: as l2 conv}, we have
\begin{align*}
    \sum_{t=2}^\infty \frac{\E\|\bsr_t\|}{t^{1-\frac{a(1+\lambda)}{2}}(\log t)^{1+\epsilon}} &\leq L_R \sum_{t=2}^\infty \frac{\E\|\bsx_t - \bsx^*\|^{1+\lambda}}{t^{1-\frac{a(1+\lambda)}{2}}(\log t)^{1+\epsilon}} \leq L_R \sum_{t=2}^\infty \frac{\left(\E\|\bsx_t - \bsx^*\|^{2}\right)^{\frac{1+\lambda}{2}}}{t^{1-\frac{a(1+\lambda)}{2}}(\log t)^{1+\epsilon}} \\
    &\lesssim \sum_{t=2}^{\infty} t^{-1}(\log t)^{-1-\epsilon} < \infty,
\end{align*}
and thus with probability one,
\begin{align*}
    \sum_{t=2}^\infty \frac{\|\bsr_t\|}{t^{1-\frac{a(1+\lambda)}{2}}(\log t)^{1+\epsilon}} < \infty.
\end{align*}
Hence, by Kronecker's lemma,
\begin{align*}
    \frac{1}{t^{1-\frac{a(1+\lambda)}{2}}(\log t)^{1+\epsilon}}\sum_{j=0}^{t-1} \|\bsr_j\| \rightarrow 0 \quad \text{a.s.} \quad \Longleftrightarrow \quad \sum_{j=0}^{t-1} \|\bsr_j\| = o\left(t^{1-\frac{a(1+\lambda)}{2}}(\log t)^{1+\epsilon}\right) \quad \text{a.s.}
\end{align*}
The convergence rate can be slightly improved to $\sum_{j=0}^{t-1} \|\bsr_j\| = o\left(t^{1-\frac{a(1+\lambda)}{2}}(\log t)(\log\log t)^{1+\epsilon}\right)$ almost surely.

\subsection{Proof of Lemma \ref{lem: bound g L2}}

When $\|\bsx_t  - \bsx^*\| \leq \delta_H$, by Assumption \ref{ass: local linear} we have
\begin{align*}
    \|\bsg(\bsx_t)\|^2 &\leq \left( \|\bsH\| \|\bsx_t - \bsx^*\| + L_H \|\bsx_t - \bsx^*\|^{1+\lambda} \right)^2 \\
    &\leq 2\|\bsH\|^2 \|\bsx_t - \bsx^*\|^2 + 2L_H^2 \|\bsx_t - \bsx^*\|^{2(1+\lambda)} \\
    &\leq (2\|\bsH\|^2 + 2L_H^2 \delta_H^{2\lambda}) \|\bsx_t - \bsx^*\|^2.
\end{align*}
When $\|\bsx_t  - \bsx^*\| > \delta_H$, by Assumption \ref{ass: noise} \ref{ass: noise 1} we have
\begin{align*}
    \|\bsg(\bsx_t)\|^2 &\leq L_\varepsilon (1 + \|\bsx_t - \bsx^*\|^2) \leq L_\varepsilon \left(1 + \frac{1}{\delta_H^2} \right) \|\bsx_t - \bsx^*\|^2.
\end{align*}
Combining the above two cases, if we let $C_G:= \max\left\{ 2\|\bsH\|^2 + 2L_H^2 \delta_H^{2\lambda}, L_\varepsilon \left(1 + \frac{1}{\delta_H^2} \right) \right\}$, then we get $\|\bsg(\bsx_t)\|^2 \leq C_G \|\bsx_t - \bsx^*\|^2$ and taking expectation concludes the proof.

\subsection{Proof of Lemma \ref{lem: as l2 conv}}

In this proof, we establish the almost-sure and $L^2$ convergence rates of $\bsx_t$.

\paragraph{Almost-sure convergence.} Recall that $\bsDelta_t = \bsx_t - \bsx^*$. Due to Assumption \ref{ass: lyapunov} \ref{ass: lyapunov 3}, the increment of the function $V_t := V(\bsDelta_t)$ on one step of \eqref{eq: sa} is given by
\begin{align*}
    V_t \leq V_{t-1} - \eta_{t-1}  \nabla V_{t-1}^\top \bsG(\bsx_{t-1}, \xi_t) + \frac{L_V}{2} \eta_{t-1}^2 \|\bsG(\bsx_{t-1}, \xi_t)\|^2.
\end{align*}
Taking conditional expectation with respect to $\gF_{t-1}$, and leveraging Assumption \ref{ass: noise} \ref{ass: noise 1} and Assumption \ref{ass: lyapunov} \ref{ass: lyapunov 2}\ref{ass: lyapunov 5}, we have
\begin{align}
    \E(V_t \mid \gF_{t-1}) &\leq V_{t-1} - \eta_{t-1}  \nabla V_{t-1}^\top \bsg(\bsx_{t-1}) + \frac{L_V}{2}\eta_{t-1}^2 \cdot 2\left( \E(\|\bsvareps_t\|^2\mid \gF_{t-1}) + \|\bsg(\bsx_{t-1})\|^2 \right) \notag \\
    &\leq V_{t-1} + \eta_{t-1}^2 L_V L_\varepsilon (1 + \|\bsDelta_{t-1}\|^2) - \eta_{t-1} \lambda_V V_{t-1} \notag \\
    &\leq V_{t-1}\left(1 + \frac{\eta_{t-1}^2 L_V L_\varepsilon}{\mu_V} \right) + \eta_{t-1}^2 L_V L_\varepsilon - \eta_{t-1} \lambda_V V_{t-1}. \label{eq: V iterate}
\end{align}
Using Lemma \ref{lem: robbins siegmund}, the Robbins-Siegmund theorem \citep{robbins1971convergence}, $V_t$ converges to a nonnegative random variable $V_\infty$ with probability one, and
\begin{align}
\label{eq: robbins result}
    \lambda_V \sum_{t=0}^\infty \eta_t V_t <\infty.
\end{align}
If $\PB(V_\infty > 0) > 0$, then the left-hand side of \eqref{eq: robbins result} would be infinite with positive probability due to the fact that $\sum_{t=0}^\infty \eta_t = \infty$ under Assumption \ref{ass: step size}, which leads to a contradiction. Therefore, we must have $V_\infty = 0$ almost surely. The almost-sure convergence of $\bsx_t$ follows from
\begin{align*}
    \|\bsx_t - \bsx^*\|^2 = \|\bsDelta_t\|^2 \leq \frac{V_t}{\mu_V} \rightarrow 0 \quad \text{a.s.}
\end{align*}

\paragraph{$L^2$ convergence.} This part follows the similar argument in \cite{su2023higrad,li2022statistical}. Taking expectation on both sides of \eqref{eq: V iterate}, 
\begin{align*}
    \frac{\E V_t}{\eta_{t-1}} \leq \frac{\eta_{t-2} \left( 1 - \eta_{t-1}\lambda_V + \frac{\eta_{t-1}^2 L_V L_\varepsilon}{\mu_V} \right)}{\eta_{t-1}} \frac{\E V_{t-1}}{\eta_{t-2}} + \eta_{t-1} L_V L_\varepsilon.
\end{align*}
Because $\eta_t\rightarrow 0$, for sufficiently large $t$, we have $\eta_{t-1} \leq \frac{\lambda_V \mu_V}{2L_VL_\varepsilon}$, and hence
\begin{align*}
    \frac{\E V_t}{\eta_{t-1}} \leq \frac{\eta_{t-2} \left( 1 - \frac{\eta_{t-1}\lambda_V}{2} \right)}{\eta_{t-1}} \frac{\E V_{t-1}}{\eta_{t-2}} + \eta_{t-1} L_V L_\varepsilon.
\end{align*}
By Lemma \ref{lem: su and zhu}, there exists some $C>0$ such that
\begin{align*}
    \sup_{t\geq 1} \frac{\E V_t}{\eta_{t-1}} < C,
\end{align*}
which immediately implies that
\begin{align*}
    \E\|\bsx_t - \bsx^*\|^2 \leq \frac{\E V_t}{\mu_V} \leq \frac{C}{\mu_V} \eta_{t-1} = \frac{C}{\mu_V} (1 + o(\eta_t)) \eta_{t} \leq \widetilde{C}_0 \eta_t.
\end{align*}

\subsection{Other Lemmas}

\begin{lemma}[Lemma 1 of \cite{polyak1992acceleration}]
\label{lem: bounded matrix}
    Let $\bsB_t := \bsI - \eta_t \bsH$ where $\bsH$ is the matrix defined in Assumption \ref{ass: local linear}, and let $\bsA_j^t := \sum_{s=j}^t \left( \prod_{i=j+1}^s \bsB_i \right) \eta_j$ for any $j\leq t$. Then under Assumption \ref{ass: step size}, there exists a universal constant $C_0>0$ such that for any $j\leq t$, $\|\bsA_j^n\| \leq C_0$. Furthermore, it holds that $\lim_{t\rightarrow \infty} \frac{1}{t} \sum_{j=0}^{t-1} \|\bsA_j^t - \bsH^{-1}\| = 0$.
\end{lemma}

\begin{lemma}[Lemma 2 of \cite{zhu2021constructing}]
\label{lem: convergence matrix}
With the same notation as in Lemma \ref{lem: bounded matrix}, under Assumption \ref{ass: step size}, it holds that $\sum_{j=0}^{t-1} \|\bsA_j^t - \bsH^{-1}\| = O(t^a)$ where $a$ is defined in Assumption \ref{ass: step size}.
\end{lemma}

\begin{lemma}[Lemma 20 of \cite{su2023higrad}]
    \label{lem: su and zhu}
    Let $c_1, c_2 > 0$ be arbitrary positive constants. Under Assumption \ref{ass: step size}, if $B_t>0$ satisfies for all $t\geq 1$,
    \begin{align*}
        B_t \leq \frac{\eta_{t-1}(1-c_1\eta_t)}{\eta_t} B_{t-1} + c_2 \eta_t,
    \end{align*}
    then $\sup_{t} B_t < \infty$.
\end{lemma}

\begin{lemma}[Robbins-Siegmund, Theorem 1 of \cite{robbins1971convergence}]
    \label{lem: robbins siegmund}
    Let $\{D_t,\beta_t,\alpha_t,\zeta_t\}_{t\geq 0}$ be nonnegative random variables that are adapted to a filtration $\{\gG_t\}_{t\geq 0}$, satisfying for all $t\geq 0$,
    \begin{align*}
        \E(D_{t+1}\mid \gG_t) \leq (1 + \beta_t) D_t + \alpha_t - \zeta_t,
    \end{align*}
    and both $\sum_{t} \beta_t < \infty$ and $\sum_t \alpha_t < \infty$ almost surely. Then, with probability one, $D_m$ converges to a nonnegative random variable $D_\infty$, and $\sum_{t} \zeta_t < \infty$.
\end{lemma}


\begin{lemma}[Almost-sure invariance principle, Theorem 4.4 of \cite{strassen1967almost}]
\label{lem: asip}
    Let $X_1,X_2,\dots$ be random variables such that $\E(X_n\mid\gF_{n-1}) = 0$ and $\E(X_n^2\mid \gF_{n-1}) < \infty$, where $\{\gF_n\}_{n\geq 0}$ is the natural filtration. Let $S_n = \sum_{i=1}^n X_i$ and $V_n = \sum_{i=1}^n \E(X_i^2\mid \gF_{i-1})$, and let $f\colon \R_+ \rightarrow \R_+$ be a positive nondecreasing function such that $f(t) \rightarrow \infty$ as $t\rightarrow \infty$ and $t^{-1} f(t)$ is nonincreasing. Suppose that $V_n \rightarrow \infty$ almost surely and
    \begin{align*}
        \sum_{n=1}^\infty \frac{\E\big[ X_n^2\indicator(X_n^2 > f(V_n))\mid \gF_{n-1} \big]}{f(V_n)} < \infty.
    \end{align*}
    Let $S$ be the random function on $\R_+\cup \{0\}$ obtained by interpolating $S_n$ at $V_n$ in such a way that $S(0)=0$ and $S$ is constant in each $[V_n, V_{n+1})$ (or alternatively, is linear in each $[V_n, V_{n+1}]$). Then there is a standard Brownian motion $B = \{B(t)\}_{t\geq 0}$ in the same (potentially enriched) probability space such that
    \begin{align*}
        S(t) = B(t) + o\big((t f(t))^{1/4}\log t\big)  \quad \text{a.s.} 
    \end{align*}
\end{lemma}

\begin{lemma}[Almost-sure invariance principle in Hilbert spaces, Theorem 2 of \cite{philipp1986note}]
\label{lem: asip hilbert}
    Let $\{X_n, \gF_n\}_{n\geq 1}$ be a martingale difference sequence with values in a separable Hilbert space $\HB$ of dimension $d\leq \infty$. Define the conditional covariance operator $\sigma_n\colon \HB \rightarrow \HB$ as $\sigma_n(u) := \E(\langle u, X_n \rangle X_n \mid \gF_{n-1})$, its trace as $\tr(\sigma_n) := \E(\|X_n\|^2\mid \gF_{n-1})$, the cumulative operator as $A_n:= \sum_{i=1}^n \sigma_i$, and its trace as $V_n := \sum_{i=1}^n \E(\|X_i\|^2\mid \gF_{i-1})$.
    Let $f \colon \R_+ \rightarrow \R_+$ be a positive nondecreasing function such that $f(t) \rightarrow \infty$ as $t\rightarrow\infty$ and $t^{-1} (\log t)^\alpha f(t)$ is nonincreasing for some $\alpha > 50d$. (If $d=\infty$ the last condition should hold for all large $\alpha$.) Suppose that $V_n \rightarrow \infty$ almost surely and
    \begin{align}
    \label{eq: asip hirbert 1}
        \sum_{n=1}^\infty \E\left[\frac{ \|X_n\|^2\indicator(\|X_n\|^2 > f(V_n))}{f(V_n)} \right] < \infty.
    \end{align}
    Moreover, suppose that there exists a covariance operator $A\colon \HB \rightarrow \HB$ and a constant $0<q\leq 1$ such that
    \begin{align}
    \label{eq: asip hirbert 2}
        \E \left[ \sup_{n\geq 1} \left( \frac{\|A_n - V_n A\|}{f(V_n)} \right)^q \right] < \infty,
    \end{align}
    where the (semi)norm $\|\cdot\|$ on a linear operator $T\colon \HB \rightarrow \HB$ is defined by $\|T\| := \sup_{u\in H, \|u\|=1} |\langle u, T(u) \rangle|$.
    Then there exists a Brownian motion $B = \{B(t)\}_{t\geq 0}$ in the same (potentially enriched) probability space with values in $\HB$ and its covariance operator as $A$ such that with probability 1,
    \begin{align*}
        \left\| \sum_{n=1}^\infty X_n \indicator(V_n \leq t) - B(t) \right\| = \left\{ \begin{aligned}
            &O\big(t^{\frac{1}{2} - \frac{q}{50d}} f(t)^{\frac{q}{50d}} \big)  && \text{if } d<\infty \\
            &o\big((t\log\log t)^{\frac{1}{2}} \big) && \text{if } d = \infty
        \end{aligned} \right. . 
    \end{align*}
\end{lemma}

\begin{lemma}[Kronecker's lemma]
\label{lem: kronecker}
Let $\{a_n\}_{n\geq 1}$ be a sequence with a convergent sum $\sum_{n=1}^\infty a_n < \infty$. Then, for any increasing sequence $\{b_n\}_{n\geq 1}$ such that $b_n>0$ and $b_n\rightarrow \infty$, it holds that
\begin{align*}
    \lim_{n\rightarrow \infty} \frac{1}{b_n} \sum_{i=1}^n b_ia_i = 0.
\end{align*}
\end{lemma}

\begin{lemma}[Burkholder's inequality]
\label{lem: burkholder}
    If $S_n = \sum_{i=1}^n X_i$ is a martingale and $1<p<\infty$, then there exist constants $c_p$ and $C_p$ depending only on $p$ such that
    \begin{align*}
        c_p \E\left|\sum_{i=1}^n X_i^2\right|^{p/2} \leq \E|S_n|^p \leq C_p \E\left| \sum_{i=1}^n X_i^2 \right|^{p/2}.
    \end{align*}
\end{lemma}

\begin{lemma}[Ville's inequality]
\label{lem: ville ineq}
Let $\{S_t\}_{n\geq 0}$ be a nonnegative supermartingale. Then for any $a>0$,
\begin{align*}
    \PB\left( \sup_{n\geq 0} S_n \geq a \right) \leq \frac{\E[S_0]}{a}.
\end{align*}
\end{lemma}

\begin{lemma}[Volume of an ellipsoid]
\label{lem: ellipsoid}
Let $B = \left\{ x\in \R^d\colon \sum_{i=1}^d x_i^2/a_i^2 \leq 1 \right\}$ be a $d$-dimensional ellipsoid. Then its volume is given by
\begin{align*}
    \Vol(B) = \frac{2}{d} \frac{\pi^{d/2}}{\Gamma(d/2)} \prod_{i=1}^n a_i.
\end{align*}
\end{lemma}

\begin{lemma}[Boundary crossing probabilities]
    Let $F$ denote any measure on $(0,\infty)$ which is finite on bounded intervals, and define functions $f\colon \R \times \R \rightarrow \R \cup \{\infty\}$ and $A\colon \R \times \R_+ \rightarrow \R \cup \{-\infty\}$ as
    \begin{align*}
        f(x,t) &:= \int_0^\infty \exp\left( xy - \frac{y^2t}{2} \right) \rd F(y), \\
        A(t,\varepsilon) &:= \inf\{ x \colon f(x, t) \geq \varepsilon \}.
    \end{align*}
\end{lemma}

\end{document}